%% file: paper.tex
\documentclass{article}
\usepackage{authblk}
\usepackage[margin=1in]{geometry}

\usepackage{hyperref}

\usepackage{natbib}

\usepackage{microtype}
\usepackage{graphicx}
\usepackage{subfig}
\usepackage{booktabs} 
\usepackage{url,hyperref}

\usepackage{amsmath}
\usepackage{amssymb}
\usepackage{mathtools}
\usepackage{amsthm}
\usepackage{balance}
\usepackage{physics}
\usepackage{thm-restate}

\usepackage[capitalize,noabbrev]{cleveref}

\theoremstyle{plain}
\newtheorem{theorem}{Theorem}[section]

\theoremstyle{definition}
\newtheorem{definition}[theorem]{Definition}

\theoremstyle{remark}

\newcommand{\indep}{\perp \!\!\! \perp}
\newcommand{\red}[1]{{\color{red}{#1}}}
\newcommand{\blue}[1]{{\color{blue}{#1}}}
\usepackage[textsize=tiny]{todonotes}

\title{Approximate Causal Effect Identification under Weak Confounding}
\author{Ziwei Jiang}
\author{Lai Wei}
\author{Murat Kocaoglu}
\affil{Elmore Family School of Electrical and Computer Engineering \protect\\ Purdue University \protect\\ \tt\{jiang622, wei429, mkocaoglu\}@purdue.edu}

\newcommand{\yrcite}[1]{\citeyearpar{#1}}
\renewcommand{\cite}[1]{\citep{#1}}

\begin{document}

\maketitle
\vspace{0.35cm}

\begin{abstract}
Causal effect estimation has been studied by many researchers when only observational data is available. Sound and complete algorithms have been developed for pointwise estimation of identifiable causal queries. For non-identifiable causal queries, researchers developed polynomial programs to estimate tight bounds on causal effect. However, these are computationally difficult to optimize for variables with large support sizes. In this paper, we analyze the effect of ``weak confounding'' on causal estimands. More specifically, under the assumption that the unobserved confounders that render a query non-identifiable have small entropy, we propose an efficient linear program to derive the upper and lower bounds of the causal effect. We show that our bounds are consistent in the sense that as the entropy of unobserved confounders goes to zero, the gap between the upper and lower bound vanishes. Finally, we conduct synthetic and real data simulations to compare our bounds with the bounds obtained by the existing work that cannot incorporate such entropy constraints and show that our bounds are  tighter for the setting with weak confounders.
\end{abstract}

\section{Introduction}
\label{sec:intro}
Estimating the causal effect has long been a question of great interest in a wide range of fields, such as marketing \cite{jung2022measuring}, healthcare \cite{lv2021causal, meilia2020review}, social science \cite{freedman2010statistical}, and machine learning \cite{pearl2019seven}. The causal relation differs from the statistical association due to the existence of unobserved confounders, variables that affect both the treatment and outcome, which create a spurious association, causing the statistical association to deviate from the true causal effect. An example study of the causal relationship between weekly exercise and cholesterol in various age groups is discussed by Glymour et al.\yrcite{glymour2016causal}. In this study, the cholesterol level is negatively correlated with the amount of weekly exercise within each age group. But if the age data is not observed, the cholesterol level appears positively correlated with the amount of exercise. This is known as Simpson's paradox, where the confounding variable of age causes the sign reversal. If the variable age is not observed in the data, the true causal effect of exercise on cholesterol is not identifiable. Numerous studies have addressed this problem in different settings \cite{rosenbaum1983central, pratt1988interpretation,  pearl2022comment}. 

It is well known that the causal effect can be estimated from observational data if we could control for the confounders, which means the confounders are included in the observational data \cite{lindley1981role, rubin1974estimating, pearl1995causal}. Tian and Pearl \yrcite{tian2002general} provide conditions for the identifiability of causal queries. 

One approach for addressing non-identifiable causal queries involves making additional untestable assumptions either on the variables or on the parametric form of the model. For example, with linear model assumption, instrumental variables can be used to estimate the Average Treatment Effect (ATE) even if unobserved confounders exist \cite{bowden1990instrumental}. However, these approaches are limited to specific settings due to the restriction of variables or the expressive power of the parametric model assumptions. Many recent studies have focused on alleviating this constraint by using machine learning models as the function for the instrumental variable \cite{singh2019kernel, xu2020learning}. 

The other type of approach makes no additional assumptions but attempts to obtain bounds for the causal effect instead of point identification. This is also known as partial identification. With instrumental variables, Kilbertus et al. \yrcite{kilbertus2020class} studied bounds on ATE without the assumption that unobserved confounders affect variables additively. Padh et al. \yrcite{padh2022stochastic} extend this idea to continuous treatments. Zhang and Bareinboim \yrcite{zhang2021bounding} developed a method to bound causal effects on continuous outcomes. Recent works using generative neural networks for partial identification in high dimensional continuous settings \cite{hu2021generative, balazadeh2022partial}.

The simplest non-identifiable setting is when two observed variables, i.e., treatment and the outcome, are confounded by some latent variables. In this paper, we focus on this setting. Tian and Pearl \yrcite{tian2000probabilities} developed tight bounds in the non-parametric setting using the observational distribution as a constraint. Li and Pearl \yrcite{li2022bounds} derived bounds with nonlinear programming with partially observed confounders, i.e., only the prior distribution of the confounder is known. 

A key challenge in causal inference is determining the strength of the confounder, which refers to the degree to which the confounder is associated with the treatment and the outcome. The stronger the association, the more likely it is that the confounder is biasing the estimate of the effect of the exposure on the outcome. Sensitivity analysis is commonly used, especially for parametric models~\cite{cinelli2019sensitivity}. Many existing studies used information theoretic quantities such as directed information \cite{etesami2014directed, quinn2015directed} and relative entropy \cite{janzing2013quantifying} as measurements of the edge strength. Researchers have used entropy to discover the causal structure in the graphs \cite{kocaoglu2017entropic, kocaoglu2020applications, compton2020entropic, compton2022entropic,  compton2023minimum}. Janzing and Sch{\"o}lkopf  \yrcite{janzing2010causal} developed a theory for causal inference based on the algorithmic independence of the Markov kernels. Vreeken and Budhathoki \yrcite{vreeken2015causal, budhathoki2018origo} extend this idea by using minimum description length for causal discovery. Another common usage of information theory in causal inference is quantifying the causal influence of variables. Ay and Polani \yrcite{ay2008information} defined information flow to measure the strength of causal effect based on the causal independence of the variables. Similar to relative entropy or mutual information, the information flow measures the independence between a set of nodes $B$ and $A$ after intervening on another set $S$. Janzing et al. \yrcite{janzing2013quantifying} studied the causal influence by quantifying the changes in the distribution of a single variable response to the intervention. Geiger et al. \yrcite{geiger2014estimating} used these measures to formulate the bounds of the confounding variables by showing that the back-door dependence should be greater or equal to the deviation between observed dependence and causal effect with some measure. 

We are interested in the problem of estimating causal effect when confounders are ``simple,'' i.e., the entropy of the confounder is small. The information passing through such confounders should not be arbitrarily large, so we should get tighter bounds on the causal effect compared to the methods that cannot utilize this side information. However, it is nontrivial to incorporate low-entropy constraints since entropy is a concave function. Enforcing small entropy as a constraint directly changes the feasible set to a non-convex set. Therefore, the problem cannot be solved directly using the existing formulations. In this paper, we address this problem by quantifying the tradeoff between the strength of the unobserved confounder measured by its entropy and the upper and lower bounds on causal effect.

The main contributions of this paper are as follows:
\begin{itemize}
    \item We formulate a novel optimization problem to efficiently estimate the bounds of causal effect using counterfactual probabilities and apply the low-entropy confounder constraint using this formulation.
    \item We examine the conditions on the entropy constraint for the optimization to yield a tighter bound. We analytically show the condition when either or both treatment and outcome are binary variables.
    \item We conducted experiments using both simulated and real-world data to test our method and demonstrate that our bound is tighter than existing approaches that are unable to incorporate entropy constraints.
\end{itemize}

\begin{figure}
    \centering
    \subfloat[DAG $\mathcal{G}$ with latent confounder]{\includegraphics[scale=0.2]{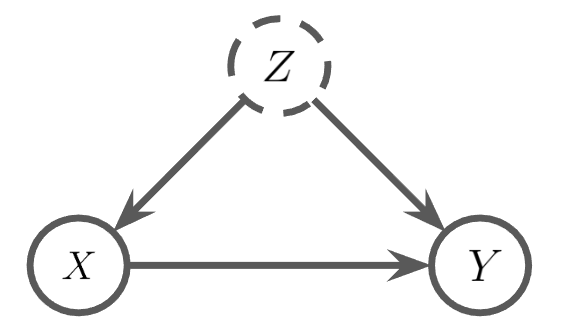}\label{subfig:dag}}
    \hfill
    \subfloat[Canonical partition of the DAG $\mathcal{G}$]{\includegraphics[scale=0.2]{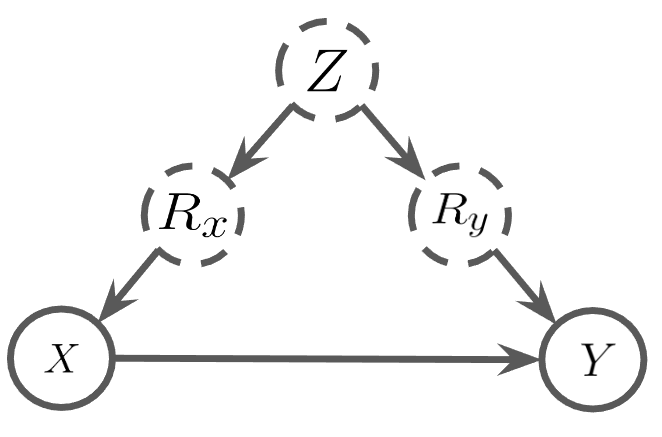}\label{subfig:cp}}
    \hfill
    \subfloat[Single world intervention graph of DAG $\mathcal{G}$]{\includegraphics[scale=0.2]{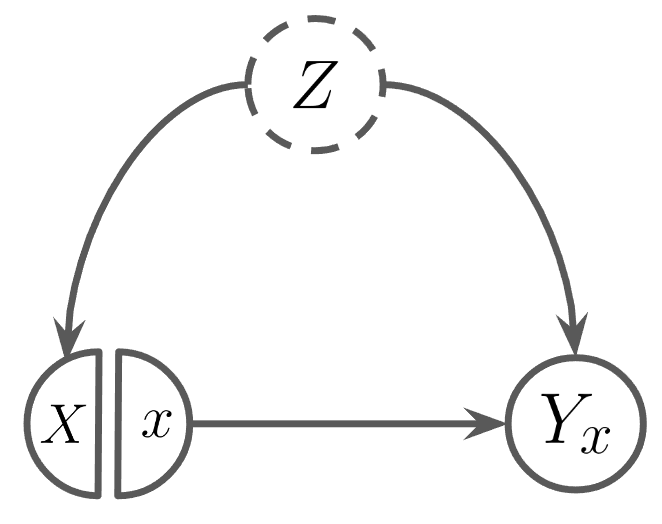}\label{subfig:cf}}
    \caption{A graph consist of treatment $X$, outcome $Y$ and an unobserved confounder $Z$ with small entropy.}
    \label{fig:latent_graph} 
\end{figure}

\section{Background and Notation}
\textbf{Notations.}
Throughout the paper, we use uppercase letters $X, Y, Z$ to denote the random variables and lowercase letters $x_i, y_i, z_i$ for their states. We use $\{x, x'\}$ to denote the states of binary variables. The Greek letters $\alpha, \beta, \theta$ are used to denote some constant value for the probability mass function or information-theoretic quantities. $\abs{X}$ represents the number of states for a random variable.
The uppercase letter with a lowercase letter as the subscript shows an intervened variable, i.e., $P(Y_x=y):= P(y|do(x))$. This notation is also used for counterfactual distributions, e.g., $P(Y_x=y| X=x')$ means the probability of $y$ had we intervened on $x$ given that $x'$ is observed. 
For a probability mass function $P(Y=y, X=x)$, we write $P(y,x)$ as an abbreviation. For counterfactual distribution $P(Y_x=y, x')$, we keep the notation of a random variable to avoid confusion.

\textbf{Entropy and mutual information.} In this paper, we use the term entropy to refer to the Shannon entropy, which quantifies the average amount of information in the variable. For a discrete random variable $X$, its Shannon entropy is defined as follows
\begin{equation*}
    H(X) = \sum_{i} -P(x_i)\log{P(x_i)}.
\end{equation*}
Mutual information is a concept closely related to entropy. It measures the average amount of information one variable carries about another. For discrete variables, $X,Y$, the mutual information is defined as
\begin{equation*}
    I(X;Y) = \sum_i \sum_j P(x_i, y_j) \log{\frac{P(x_i, y_j)}{P(x_i)P(y_j)}}.
\end{equation*}

\textbf{Causal DAG.}
A directed acyclic graph (DAG) encodes the causal relationship between variables, where the nodes represent  variables, and the edges represent their causal relationships. Causal graphs are often used to help identify and understand complex systems. A graphical condition called d-separation can  be used to read-off the independence induced by a graph. Pearl \yrcite{pearl1995causal} introduced a set of rules called do-calculus for deriving causal queries from observational data. 

\textbf{Structural Causal Model.}
Pearl \yrcite{pearl1995causal} introduced the Structural Causal Model (SCM), a mathematical framework that can be used to study causality and counterfactuals \cite{pearl2009causality, zhang2022partial}. We can describe causal relationships between variables with a set of functions in SCM. More specifically, an SCM is a tuple $\{\mathcal{U}, \mathcal{V}, \mathcal{F},\mathbb{P}\}$ where $\mathcal{U}$ is a set of exogenous variables, $\mathcal{V}$ is a set of endogenous variables, and $\mathcal{F}$ is a set of functions with same cardinality of $\mathcal{V}$, and $\mathbb{P}$ is a product probability measure. Each $v\in\mathcal{V}$  is generated by some $f\in\mathcal{F}$ as a function of other variables. The functions in an SCM impose a causal graph. 

\textbf{Canonical Partition.}
A widely used method for bounding the causal effects is canonical partition. Consider an SCM with the graph in \cref{subfig:dag}. The binary variables $X, Y$ are generated by the functions $y = f_y(x,u_y)$ and $x=f_x(u_x)$. The latent variables $U_x$ and $U_y$ are not independent since there is latent confounding between $X$ and $Y$. Balke and Pearl \yrcite{balke1997bounds} pointed out that the latent variable $U_y$ can be replaced by a finite-state response variable $R_y$ representing the distinct functions mapping $X$ to $Y$. All these functions can be represented by a response variable $R_y$ with four states.
\begin{table}[hb]
\caption{Response variable $R_y$}
\label{table:cp_intro}
\vskip 0.15in
\begin{center}
\begin{small}
\begin{sc}
\begin{tabular}{l c|c|c r}
\toprule
    & & $X=x_0$ & $X=x_1$\\
    \midrule
    &$R_y=0$ &$Y=y_0$ & $Y=y_0$ \\
    &$R_y=1$ &$Y=y_0$ & $Y=y_1$ \\
    &$R_y=2$ &$Y=y_1$ & $Y=y_0$\\
    &$R_y=3$ &$Y=y_1$ & $Y=y_1$\\
\bottomrule
\end{tabular}
\end{sc}
\end{small}
\end{center}
\end{table}

Each row corresponds to a function that maps $X$ to $Y$. For example if $R_y = 1$, we have
$$Y = f_Y(X, R_y=1) = \begin{cases}
    y_0  & \text{if } X=x_0 \\
    y_1 & \text{if } X=x_1.
\end{cases}    $$

And for $X$ we have $R_x \in \{1,2\}$
\begin{align*}
    &R_x=0: X=x_0, \\
    &R_x=1: X=x_1.
\end{align*}

Due to the latent confounder, $R_x$ and $R_y$ are dependent as shown in \cref{subfig:cp}. The joint distribution $P(R_x, R_y)$ has total of 8 states, denoted by $q_{ij} = P(R_x=i, R_y=j),$ and $ p_{ij} = P(x_i, y_j)$. The observational probability can be expressed as $p_{00} = q_{00}+q_{01}$, $p_{01} = q_{02}+q_{03}$, $p_{10} = q_{10}+q_{12}$, and $p_{11} = q_{11}+q_{13}$.

The causal effect $P(y_0|do(x_0))$ is the probability of the function which maps $x_0$ to $y_0$. Therefore we have $P(y_0|do(x_0)) = q_{00}+q_{01}+q_{10}+q_{11}$. Combining the above equations, we obtain the bounds of causal effect in a closed-form expression: $p_{00}\leq p(y_0|do(x_0))\leq1-p_{01}$. This method has been used in causal inference problems. Tian and Pearl \yrcite{tian2000probabilities} apply this to estimate the bounds for the probability of causation given the interventional data. Zhang and Bareinboim \yrcite{zhang2017transfer} use this method to derive bounds for the multi-arm bandit problem. 

This bound holds true for any pair of variables, so it can not incorporate any side information, such as the graph structure or the prior distribution of the unobserved confounder, and might be loose in some cases. For example, consider the distribution $P(X, Y)$ with binary $X$ and $Y$ and a low entropy confounder. Suppose both $P(x, y)$ and $P(x, y')$ are small, and we know that the entropy of the confounder is small, i.e., upper bounded by some small value $\theta$. Intuitively, the causal effect $p(y|do(x))$ should be close to $P(y|x)$. This is validated with experiment in \cref{subsec:random}. Without incorporating this information about confounder, the bounds are not very informative since the bounds are close to $0$ and $1$:
\begin{equation*}
    0 \approx P(x, y) \leq P(y|do(x))   \leq 1-P(x, y') \approx 1
\end{equation*}

In \cref{sec:bound_partition}, we form an optimization problem that can incorporate such low entropy constraints.

\textbf{Counterfactual and Single-World Intervention Graph (SWIG).} 
Counterfactual queries are questions of the form \emph{``What would happen if an intervention or action had been taken differently, given what already has happened.''} Pearl \yrcite{pearl2009causality} introduced counterfactual reasoning with the SCM. A counterfactual query $P(Y_x=y|x')$ reads, ``The probability of $y$ had we intervened on $x$ given $x'$ is observed." In general, given an SCM, the counterfactual queries can be estimated with three steps: ``abduction,'' ``action,'' and ``prediction.'' The first step is to use the observed $x'$ as evidence to update the exogenous variables $U$. The second step is to apply the intervention by replacing the value in the SCM with $x$. And lastly, make predictions with the updated SCM. 

Richardson and Robins \yrcite{richardson2013single} introduced a graphical representation to link the counterfactual distribution and DAG, called Single World intervention graphs (SWIGs). We can represent the interventional variable $Y_x$ as a node in the DAG and split the treatment variable into nodes $X$ and $X=x$. As shown in \cref{subfig:cf}, we have $Y_x$ independent from $X$ given $Z$.

\section{Bounding Causal Effect with Entropy Constraint}\label{sec:optimization}
\subsection{Bounds with Canonical Partition}\label{sec:bound_partition}

Consider the DAG in \cref{subfig:dag}, the latent factors can be represented by a joint distribution $P(R_x,R_y)$ with $\abs{R_x}=2, \abs{R_y}=4$. The canonical partition representation parameterizes the exogenous variables with $R_x, R_y$. Therefore the entropy of $H(R_y, R_x)$ does not necessarily reflect the confoundedness of $X, Y$. For example, we could have a small entropy confounder and $Y$ with large exogenous entropy. In that case,  $H(Z) \ll H(R_y)$. Since the unobserved confounders do not appear in the canonical partition representation, applying the entropy constraint for estimating bounds is not straightforward.

To overcome this difficulty, we notice that the causal effect $P(y|do(x))$ is equal to $P(y, x)+ \alpha P(y, x') +\beta P(y', x')$ for some unknown $\alpha, \beta$. Intuitively, these two parameters can be thought of as the proportion of $p(y,x')$ and $p(y',x)$ that is generated by the function that maps $x$ to $y$, i.e., from $R_y=\{0,1\}$. The causal effect attains the Tian-Pearl lower bound if $\alpha=\beta=0$. In that case, $P(R_x=1, R_y=0) = P(R_x=1, R_y=1) = 0$. This imposes a constraint on the minimum value of mutual information to attain such distribution. Since the $R_x, R_y$ are conditionally independent of $Z$, the mutual information $I(R_x; R_y)$ is bounded by the entropy of $Z$. 
Under the assumption that the confounder $Z$ is simple, i.e., $H(Z)\leq\theta$. We can apply the entropy constraint to estimate the bounds of the causal effect.

\begin{table}[bht]
\caption{Table for the counterfactual distribution}
\label{table:cf_conditional}
\vskip 0.15in
\begin{center}
\begin{footnotesize}
\begin{sc}
\raisebox{-1cm}{\rotatebox{90}{$P(Y_{x_q}|X)$}}\hspace{0pt}
\begin{tabular}{l |c c c c c c c r}
\toprule
    &  \multicolumn{5}{c}{$P(x)$}\\
    & $x_0$  & $\dots$ & $x_q$ & $\dots$ & $x_n$\\
    \midrule
    $y_0$ &$\blue{b_{00}}$& $\blue{\dots}$ & $P(y_0|x_q)$ & $\blue{\dots}$ &$ \blue{b_{0n}}$ \\
    $\blue{\vdots}$ & $\blue{\vdots}$  & $\blue{\ddots}$  & $\blue{\vdots}$ & $\blue{\ddots}$ & $\blue{\vdots}$ \\
    $y_p$ & $\red{b_{p0}}$ & \red{$\dots$} & $P(y_p|x_q)$ &  \red{$\dots$}& $\red{b_{pn}}$\\
    $\blue{\vdots}$ & $\blue{\vdots}$ & $\blue{\ddots}$ & $\blue{\vdots}$ & $\blue{\ddots}$ &$\blue{\vdots}$  \\
    $y_m$ & $\blue{b_{m0}}$ & $\blue{\dots}$ & $P(y_m|x_q)$ & $\blue{\dots}$ & $\blue{b_{mn}}$\\
\bottomrule
\end{tabular}
\end{sc}
\end{footnotesize}
\end{center}
\vskip -0.1in
\end{table}

\begin{restatable}{theorem}{cpalg}
\label{thm:cpalg}
Let $(X, Y)$ be the pair of variables in the causal graph in \cref{subfig:dag} with the joint distribution $P(X,Y)$. Suppose $\abs{X}=n, \abs{Y}=m$. Assuming $X$ and $Y$ are confounded by a set of small entropy unobserved variables $Z$, i.e., $H(Z)\leq \theta$ for some $\theta\in \mathbb{R}$. The causal effect of $x_q$ on $y_p$ is bounded by $\text{LB} \leq P(y_p|do(x_q)) \leq \text{UB}$, where
\begin{align*}
    \text{LB}/\text{UB} &= \min/\max{\left(\sum_{i=0}^{m^n-1}\sum_{j=0}^{n-1} a_{ij}P(x_j)\right)}\\
    \text{s.t.}&\:
    \sum_{i,j} a_{ij}P(x_j) = 1;
    \sum_{i=0}^{m^{n-1}}  a_{iq}P(x_q) = P(y_p, x_q);
    0\leq a_{ij} \leq 1\; \forall i,j;\\
    &\sum_{i,j}  a_{ij}P(x_j) \log\left(\frac{a_{ij}}{\sum_k a_{ik}P(x_k)}\right)\leq \theta.
\end{align*}    
\end{restatable}

We formulate the bounds as optimization problems with entropy constraints. $a_{ij}$ is the parameter for the optimization problem. 
The canonical partition can be naturally generalized to variables in higher dimensions. However, the number of states in the optimization problem quickly becomes intractable with the number of states of the observed variables. For $|X| = n$, $|Y|=m$, the number of possible functions mapping $X$ to $Y$ is $m^n$, so $|R_y| = m^n$. The total number of parameters is $nm^n$, which grows exponentially fast as the number of states of $X$ and $Y$ increase. In the next subsection, we present an alternative formulation to estimate bounds. 

\subsection{Bounds via Counterfactual Probabilities} 
We propose a new optimization problem using counterfactual probabilities to address the computational challenge in the canonical partition method.

For the causal graph in \cref{subfig:dag}, the interventional distribution can be represented as $P(Y_x) = P(Y_x, x) + P(Y_x, x')$.
By the consistensy property \cite{robins1987graphical}, we have $P(Y_x, x) = P(Y, x) $. And by the axiom of probability, $P(y_x, x') \leq P(x')$ for any $y\in Y$. 
\begin{align}
\label{eq:relate}
   P(y, x)  &\leq P(Y_x= y) \nonumber \\
   &= P(y,x) + P(y_x, x') \nonumber\\
   &\leq  P(y, x) + P(x') \\
   &= 1 - P(y', x) \nonumber
\end{align}
The above derivation shows the bounds from the counterfactual probability are equivalent to Tian-Pearl bounds. $Y_x$ and $X$ are d-separated by the confounder $Z$, i.e. $Y_x\indep X | Z $.

\begin{figure}[t]
    \centering
    \subfloat[The entropy threshold]
    {\includegraphics[width=0.4\linewidth]{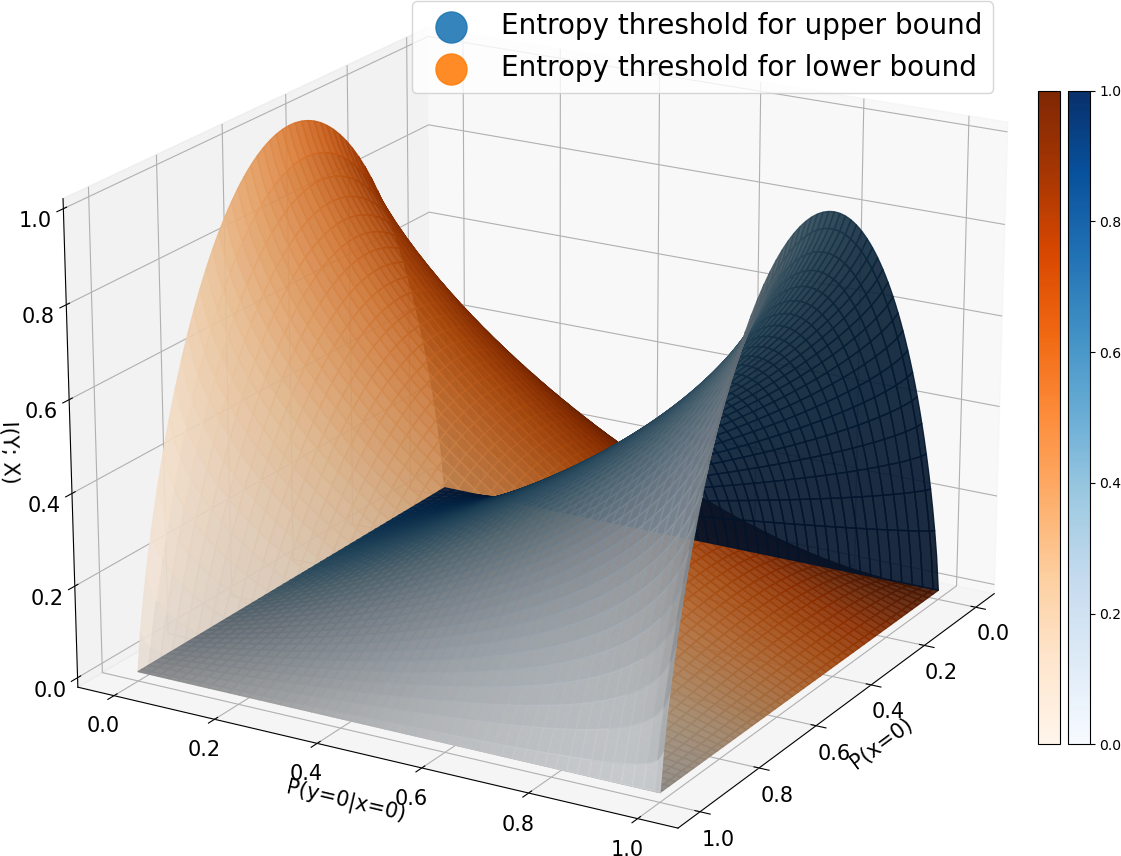}
    \label{fig:required_MI}}
    \hspace{0.1\linewidth}
    \subfloat[The average gap of bounds]{\includegraphics[width=0.4\linewidth]{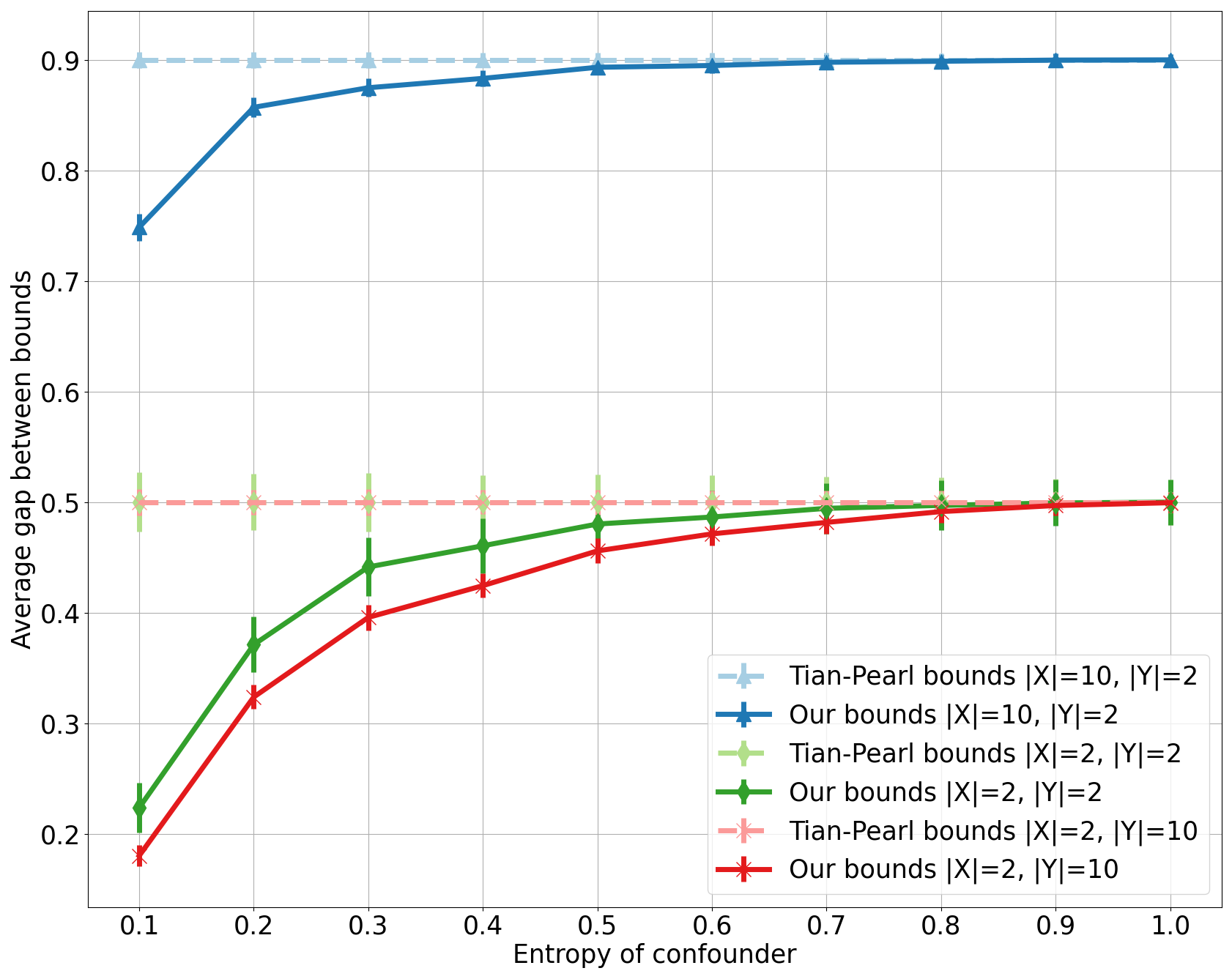}
    \label{fig:averagebounds}}
    \caption{(a) The entropy thresholds are obtained by sampling $P(x_0)$ and $P(y_0|x_0)$ from $0$ to $1$ which are the $x$ and $y$ axies in the figure. The orange surface represents the entropy threshold for obtaining a tighter upper bound; the blue surface represents the entropy threshold for obtaining a tighter lower bound. The lightness indicates the gap between the upper and lower bound; the lighter the color, the smaller the gap. Without entropy constraint, the gap depends on the value of $P(x_0)$. (b) The x-axis represents the entropy groups, and the y-axis represents the average gap in the group. The error bars represent the $95\%$ confidence interval.}
\end{figure}

\begin{figure*}[htb]
    \centering
    \includegraphics[scale=0.32]{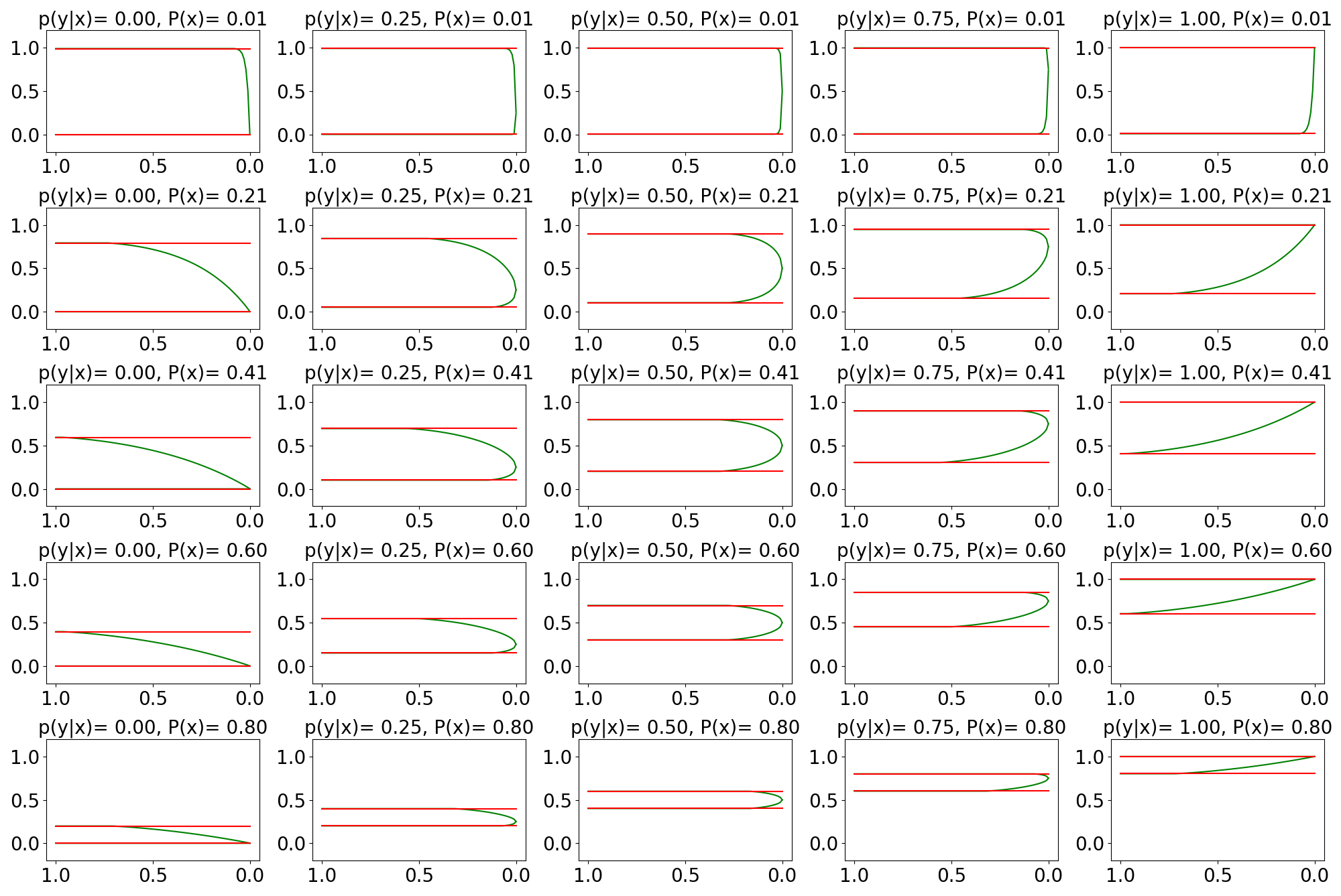}
    \caption{Bounds of the causal effect. The red lines show Tian-Pearl’s bounds, and the green lines show our bounds. The x-axis represents the entropy constraint, and the y-axis represents the causal effect $P(y|do(x))$. For each row $P(y|x)$ increases as $P(x)$ is fixed; $P(x)$ increases from top to bottom. The gap between the upper and lower bound decreases monotonically as $P(x)$ increases. The entropy threshold is high when $P(x)$ is close to $0.5$ and $P(y|x)$ is close to $1$ or $0$. }
    \label{fig:conf_x2y2}
\end{figure*}

Similar to the argument in \cref{sec:bound_partition}, a minimum value of mutual information $I(Y_x; X)$ exists for the causal effect to attain maximum/minimum. By exploiting the d-separation in the SWIG, we can impose the entropy constraint for the optimization problem. We present an optimization problem with entropy constraint based on this method and show that this formulation significantly reduces the number of parameters compared to the canonical partition approach.
\begin{restatable}{theorem}{cfalg}
    \label{thm:cfalg}
Let $(X, Y)$ be the pair of variables in the causal graph in \cref{subfig:dag} with the joint distribution $P(X,Y)$. Suppose $\abs{X}=n, \abs{Y}=m$. Assuming $X$ and $Y$ are confounded by a set of small entropy unobserved variables $Z$,  i.e., $H(Z)\leq \theta$ for some $\theta\in \mathbb{R}$. The causal effect of $x_q$ on $y_p$ is bounded by $\text{LB} \leq P(y_p|do(x_q)) \leq \text{UB}$, where
\begin{align*}
    \text{LB}/\text{UB} &= \min/\max{\left(\sum_{j} b_{pj}P(x_j)\right)}\\
    \text{s.t.}& \:\sum_{i,j} b_{ij}P(x_j) = 1;
     b_{iq}P(x_q) = P(y_i, x_q)\; \forall i ;
     0\leq b_{ij}\leq 1 \; \forall i,j;\\
     &\sum_{i,j} b_{ij}P(x_j) \log\left(\frac{b_{ij}}{\sum_k b_{ik}P(x_k)}\right) \leq \theta.
\end{align*}
\end{restatable}
Here $b_{ij}$ are the parameters for the optimization problem. Similar to \cref{sec:bound_partition}, we form the causal effect bounds estimation as a maximization and minimization problem. This formulation is more efficient than the canonical partition method in terms of the number of parameters. Consider again the case $|X| = n$, $|Y|=m$; the number of parameters for this optimization problem is $nm$. This number is significantly smaller than the canonical partition case with $nm^n$ parameters. We will discuss this in more detail in \cref{subsec:compare}.

\section{Condition for Obtaining Tighter Bounds} 

For \cref{thm:cpalg}, the mutual information $I(R_y; X)$ is upper bounded by $\theta$. And for \cref{thm:cfalg}, mutual information $I(Y_x; X)$ is upper bounded by $\theta$. In both formulations, the entropy constraint depends on the mutual information between $X$ and another variable. 
The bounds with entropy constraint will be identical to Tian-Pearl bounds when the upper bound on the confounders entropy is large. We define the greatest value of entropy constraint that yields tighter bounds as the ``entropy threshold''. 

{\definition Let $(X, Y)$ the pair of variables in the causal graph in \cref{subfig:dag}. Given an observational distribution $P(X, Y)$ and a causal query $P(y_p|do(x_q))$, the entropy threshold is the greatest entropy constraint such that the bounds obtained from \cref{thm:cfalg} are tighter than the Tian-Pearl bounds.}

The entropy threshold depends on the observational distribution $P(X, Y)$. The following lemmas show the entropy threshold when either $X$ and $Y$ are binary variables.

\begin{restatable}{lemma}{binaryxy}
    \label{lm:binaryxy}
Let $(X, Y)$ be the pair of binary variables in the causal graph in \cref{subfig:dag}. Consider $P(Y_x, X)$ for any $x\in X$. Assume, without loss of generality, $P(y|x) \geq P(y'|x)$. Then the following conditions are equivalent: 
\begin{enumerate}
    \item $P(Y_x=y)$ attain the Tian-Pearl lower bound,
    \item $P(Y_x=y')$ attain the Tian-Pearl upper bound,
    \item $I(Y_x;X)$ is maximized for the given $P(X,Y)$.
\end{enumerate}

\end{restatable}

\begin{restatable}{lemma}{binaryx}
    \label{lm:binaryx}
Let $(X, Y)$ be the pair of variables in the causal graph in \cref{subfig:dag}, where $\abs{X}=2$ and $\abs{Y} = m$. The causal effect $P(Y_x=y_p)$ attain the Tian-Pearl upper bound when $P(Y_x=y_p|x')=1$; attain the Tian-Pearl lower bound with minimum mutual information when $P(Y_x=y_i|x') = \frac{P(Y_x=y_i|X=x)}{\sum_{j\neq p}P(Y=y_j|X=x)}$ for all $i\neq p$.
\end{restatable}

\begin{restatable}{lemma}{binaryy}
    \label{lm:binaryy}
Let $(X, Y)$ be the pair of variables in the causal graph in \cref{subfig:dag}, where $\abs{Y}=2$ and $\abs{X}=n$. The causal effect $P(y|do(x_q))$ attain the Tian-Pearl upper bound when $P(Y_{x_q}=y|x_j)=1, \forall j \neq q$; attain the Tian-Pearl lower bound when $P(Y_{x_q}=y|x_j)=0, \forall j \neq q$.
\end{restatable}

The above lemmas build the link between the bounds of causal effect and the mutual information of counterfactual distribution. One can think of $b_{ij}$ as unknown conditional probabilities as shown in \cref{table:cf_conditional}. The $q$-th column (black) equals the conditional distribution $P(Y|x_q)$, which serves as the constraint from observational distribution. The causal effect is maximized when all entries in $p$-th row (red) are equal to one; minimized when they are equal to zero.

Next, the following theorem shows the relation between observational distribution $P(X, Y)$ and the entropy threshold.

\begin{restatable}{theorem}{entropyconstraint}
    \label{thm:entropyconstraint}
    Let $(X, Y)$ be a pair of variables in a causal graph $G$ as shown in \cref{subfig:dag}, where either $X$ or $Y$ is binary.  Let $(U,V)$ be two binary variables such that $P(v_0|u_0) = P(y_p|x_q)$, $P(v_1|u_0) = 1-P(y_p|x_q)$, and $P(u_0) = P(x_q)$. The entropy threshold for the bounds of $P(y_p| do(x_q))$ is equal to $\max(I(U; V))$.
\end{restatable}

By \cref{thm:entropyconstraint}, we can compute the entropy threshold for a given distribution $P(X, Y)$. Then if we know that the confounder is simple, i.e., with entropy less than the threshold, we can use the entropy constraint to obtain a tighter bound.

\cref{fig:required_MI} shows the entropy threshold for different value of $P(x)$ and $P(y|x)$. The entropy threshold is higher when $P(x)$ is close to $0.5$. For fixed $P(x)$, the threshold increases as $P(y|x)$ is close to $0$ or $1$, which corresponds to the causal effect's lower and upper bound. Without entropy constraint, the gap between bounds is only related to $P(x)$.

Following from \cref{thm:entropyconstraint}, the entropy threshold of $P(y_p|do(x_q))$ only depends on the value of $P(x_q)$ and $P(y_p|x_q)$. So we sample $P(x)$ from $0.01$ to $0.8$ and $P(y|x)$ from $0$ to $1$. Then let the $p(Y|x')$ be uniform distributions.
For each pair of $p(x)$ and $p(y|x)$, we calculate the bounds with entropy constraint for each distribution from $1$ to $0$. The results in \cref{fig:conf_x2y2} demonstrate the gap between our bounds vanishes as entropy goes to 0. The entropy threshold is small when $P(x)$ is close to $0$ or $1$ and $P(y|x)$ is close to $0.5$. On the other hand, the entropy threshold is high when $P(x)$ is close to $0.5$ and $P(y|x)$ is close to $0$ or $1$. For a fixed conditional probability and entropy constraint, the gap between bounds decreases monotonically with $P(x)$.

\section{Experiments}\label{sec:experiment}
We demonstrated our method with simulated and real-world datasets in this section. First, we show the behavior of the bounds with randomly sampled distributions $P(X, Y)$. We change the entropy constraint $\theta$ from $1$ to $0$ for each sampled distribution. We also experiment with the full distribution $P(X, Y, Z)$ where $Z$ is the low entropy confounder and $X, Y$ in high dimensions.
We show the experimental results with the real-world dataset Adult \cite{Dua:2019}. Since our algorithm works for discrete random variables with binary treatment or outcome, we take a subset of features in the graph and modify some features by discretizing continuous variables or combining states with very low probabilities. And finally, we experiment with our method in the finite sample setting and compare two optimization problem formulations.

\begin{figure*}[t!]
    \centering
     \subfloat[$\abs{X}=2, \abs{Y}=2$]{\includegraphics[width=0.33\linewidth]{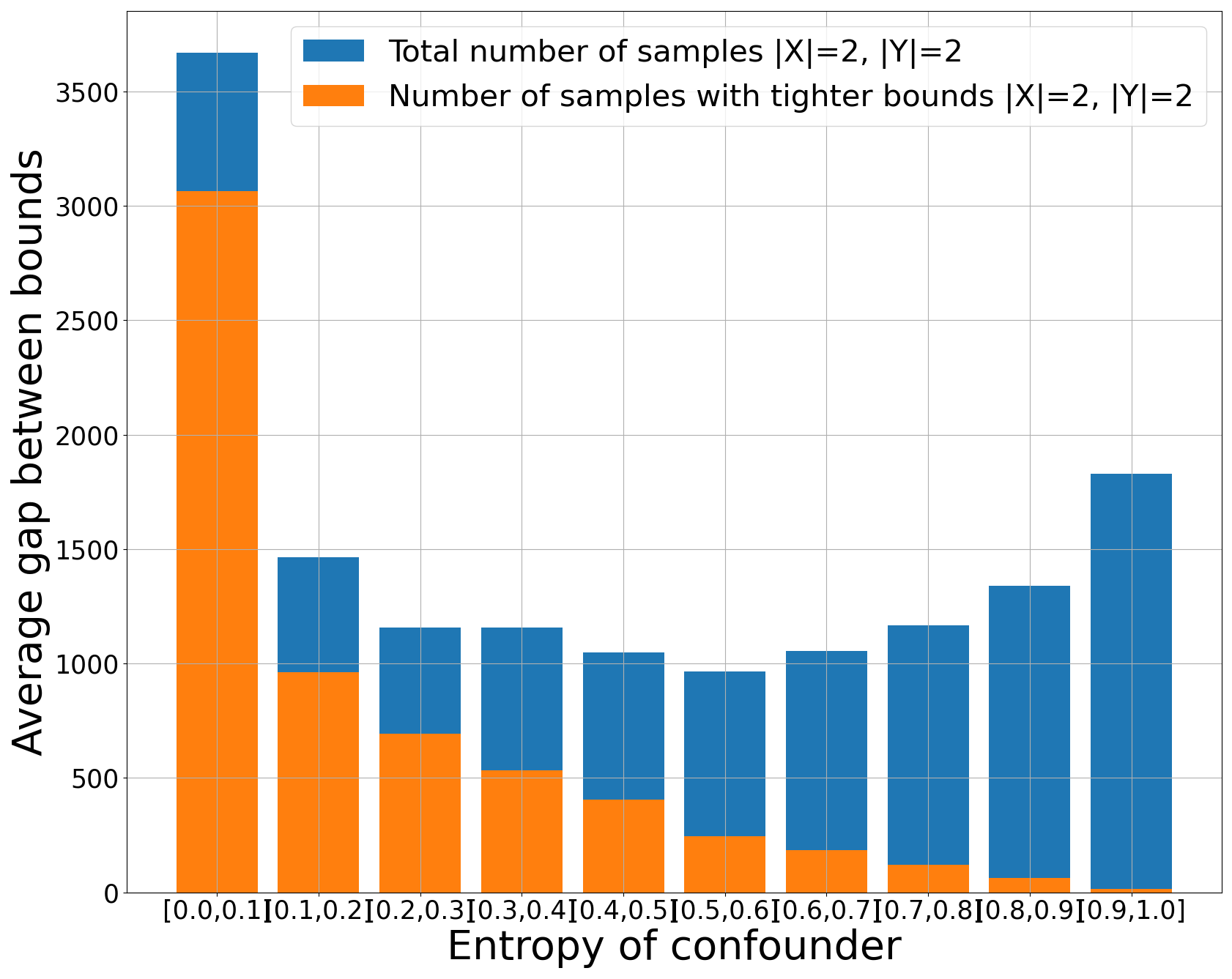}\label{subfig:x2y2}}
    \hfill
    \subfloat[$\abs{X}=2, \abs{Y}=10$]{\includegraphics[width=0.33\linewidth]{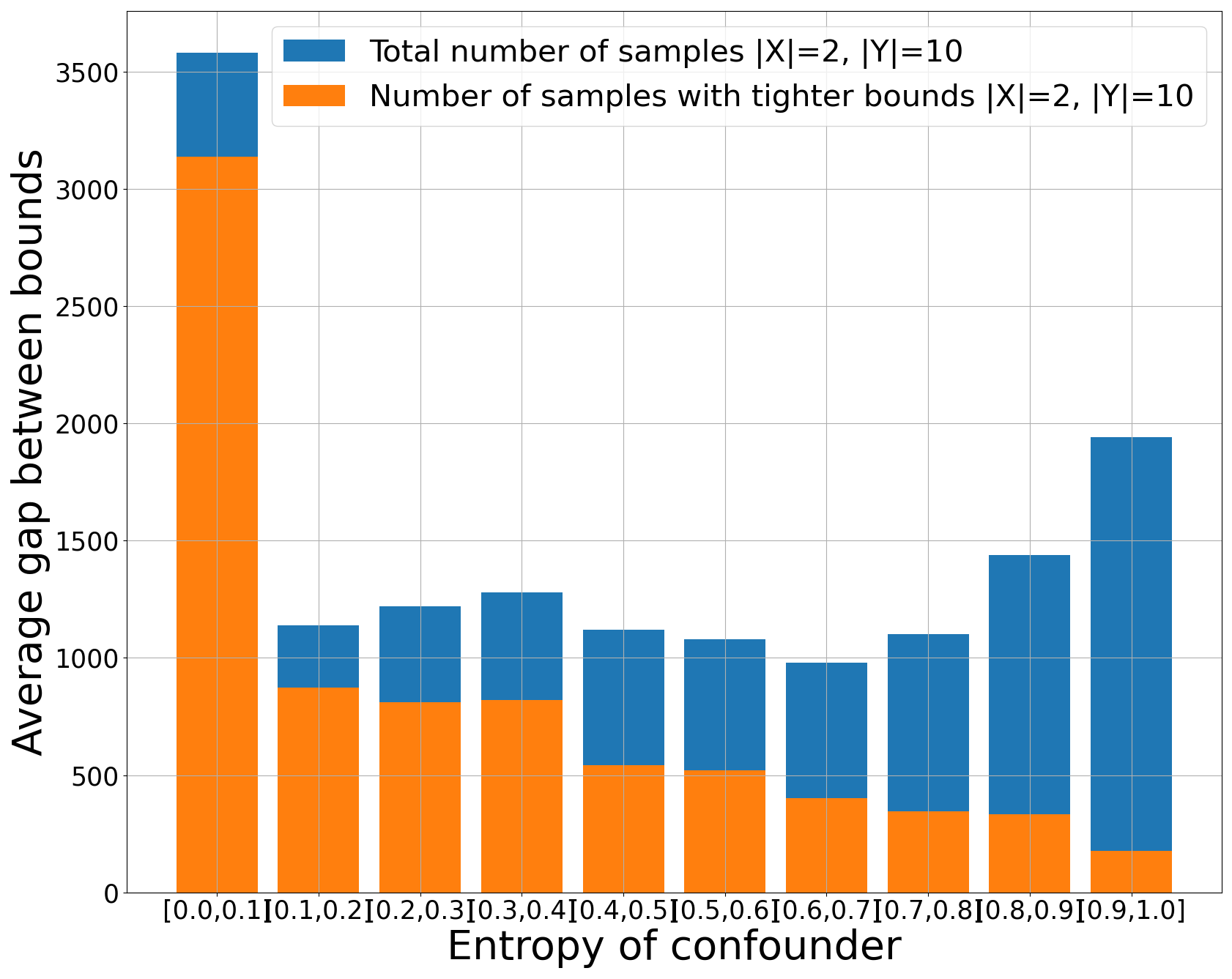}\label{subfig:x2y10}}
   \subfloat[$\abs{X}=10, \abs{Y}=2$]{\includegraphics[width=0.33\linewidth]{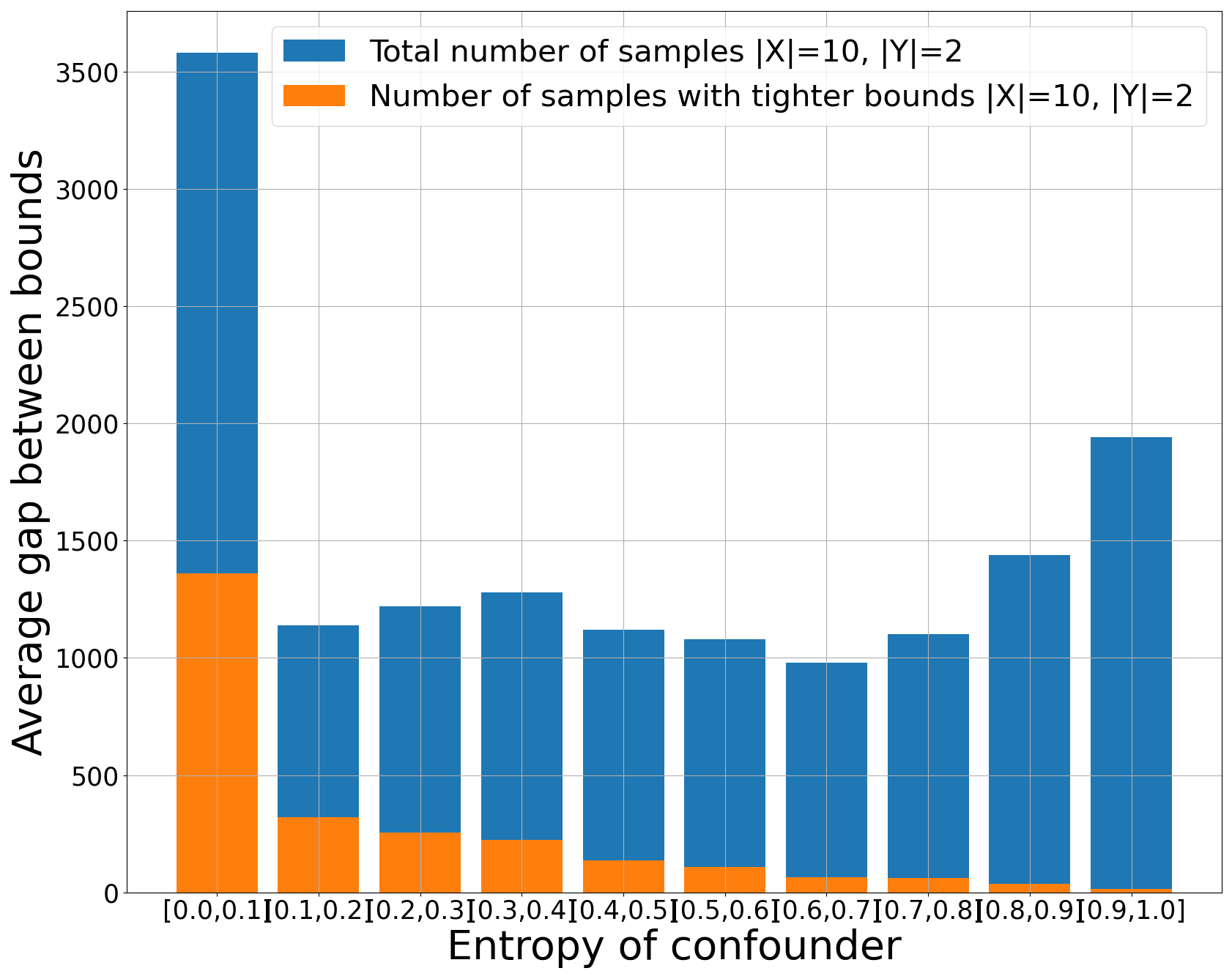}\label{subfig:x10y2}}
    \caption{The number of samples with tighter bounds. The blue bars represent the total number of distributions in each group and the orange bars show the number of distributions with tighter bounds.}
    \label{fig:tight_counts}
\end{figure*}

\subsection{Performance on Randomly Sampled Distributions}
\label{subsec:random}

First, we want to compare the gaps of our bounds and Tian-Pearl bound. We use randomly sampled data to compare the gaps. We sample the full joint distribution $P(X,Y,Z)$ according to the \cref{subfig:dag}. Then we treat $P(X, Y)$ as observational data and variable $Z$ as the unobserved confounder and estimate the causal effect using the entropy of $Z$ as $\theta$. The details for sampling the full distribution are in \cref{sec:sampling}. 
We tested three cases: $\left(\abs{X}=2,\abs{Y}=2\right), \left(\abs{X}=2,\abs{Y}=10\right)$ and $\left(\abs{X}=10,\abs{Y}=2\right)$. For each case, we generate $20000$ distributions and compute the bounds $P(y_i|do(x_j))$ for each pair of $(i,j)$. The result is shown in \cref{fig:averagebounds}. To enhance the interpretability of the results, we group the samples based on the entropy of the confounder and compare the average gap for each entropy group. For example, we consider entropy ranges such as $H(Z)\in [0,0.1)$, $[0.1,0.2)$, and so on. Notice that the average gap is smaller when $X$ is binary. This is mainly because a larger portion of samples with $P(x)$ close to $0.5$ has a larger entropy threshold. We demonstrate this by plotting the number of samples that yields a tighter bound in \cref{fig:tight_counts}. When $\abs{X}$ is large, it is less likely to have $P(x)$ close to $0.5$, and as the \cref{fig:required_MI} shows, the entropy threshold is low when $P(x)$ is close to $0$ or $1$, so there is a small number of distributions yields tighter bounds as shown in \cref{subfig:x10y2}. On the other hand, when $\abs{X}=2$, it is more likely to obtain $P(x)$ that close to $0.5$ while $P(y|x)$ is close to the boundary. So the entropy threshold is higher on average, and more distributions with tighter bounds are shown in \cref{subfig:x2y2} and \cref{subfig:x2y10}.

Next, we will consider experiments in a more realistic setting and see how the entropy constraint could be useful in the real-world problem of causal inference.

\begin{figure}
    \centering
    \subfloat[ADULT]{\includegraphics[width=0.35\linewidth]{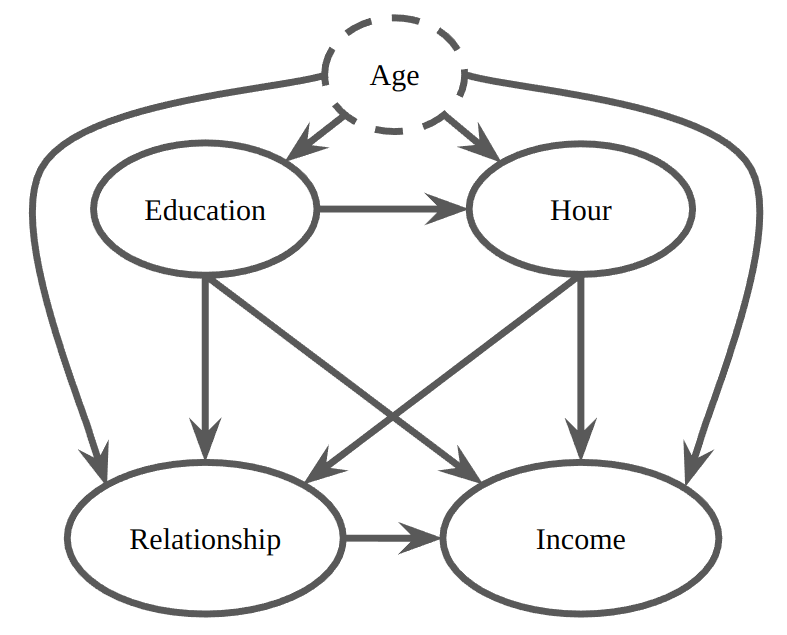}\label{subfig:adult}}
    \hspace{0.2\linewidth}
    \subfloat[INSURANCE]{\includegraphics[width=0.25\linewidth]{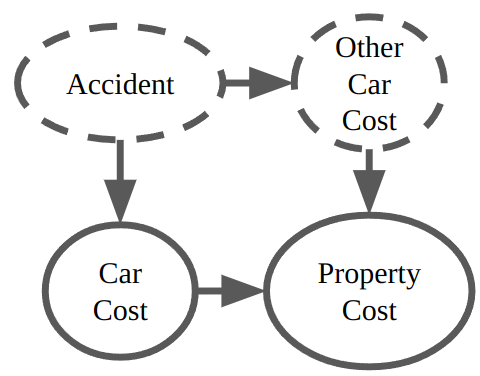}\label{subfig:insurance}} 
    \caption{Causal graphs for the real-world experiments. Only a subset of nodes in the dataset is shown in the figure. Unused variables are omitted.}
    \label{fig:real_exp}
\end{figure}

\subsection{Real-World Dataset Experiments}

\input{adult_table}

In this section, we experiment with the INSURANCE dataset \cite{binder1997adaptive} and the ADULT dataset \cite{Dua:2019}.

For the INSURANCE dataset, we aim to estimate the causal effect of Car Cost on the expected claim of the Property Cost. We consider the variable Accident as an unobserved variable with known entropy. The Car Cost and Property Cost claim are confounded through the cost of the other car in the accident as shown in \cref{subfig:insurance}. The results in \cref{table:real_exp} indicate narrow bounds on the causal effect when the entropy of the confounder is small. Therefore, we can have confidence in predicting the expected claim based on car cost, even in the presence of the confounding variable. 

For the ADULT Dataset \cite{Dua:2019}, we take a subset of variables from the dataset with the causal graph as shown in \cref{subfig:adult}. In this experiment, we treat age as a protected feature, which may not be accessible from the dataset, and only the entropy of age is known. If we assume age not having a too complex effect on other variables, i.e., the causal effects of any variable to the income is not much different for groups of people under $65$ on average; and similarly for groups of people above $65$. The above assumption  enables us to discretize the age variable into two categories: ``young'' and ``senior'', using a cutting point of 65. Since there are other confounding variables between cause and effect, we take the conditional joint distribution as the subgroup and compute the bounds. Some of the results are summarized in \cref{table:real_exp}. One way to interpret the results is to determine whether the causal effect is positive or negative. Our tighter bounds can help in establishing a positive causal effect by comparing the lower bound of $P(Y=1|do(X=1))$ with the upper bound of $P(Y=1|do(X=0))$. Similarly, for a negative causal effect, we would compare the upper bound of $P(Y=1|do(X=1))$ with the lower bound of $P(Y=1|do(X=0))$. For instance, in \cref{table:real_exp}, for the subgroup of the population with high school or higher education and full-time jobs, the relationship has a positive effect on income. This can be seen by comparing the lower bound of $P(Income>50K|do(relationship=1))$ and the upper bound of $P(Income>50K|do(relationship=0))$. Our method of bounding causal effect can be used in decision-making processes involving such scenarios. 

 In the real-world setting, we could use expert knowledge for the complexity of confounders. Even if the confounder has many states, we could still assume the confounder has small entropy if we know many of these states may have a similar effect on the outcome.


\subsection{Finite Sample Experiments}
In this section, we conducted experiments using our method in the finite data regime, aiming to estimate bounds from finite data samples. We test cases where $\left(\abs{X}=2,\abs{Y}=2\right)$, $ \left(\abs{X}=2,\abs{Y}=10\right)$, and $\left(\abs{X}=10,\abs{Y}=2\right)$. Similar to \cref{subsec:random}, we generate 1000 distributions for each case. We use $\{10, 10^2, 10^3, 10^4\}$ samples from each distribution and compute the bounds of casual effect $P(y_i|do(x_j))$ for each pair of $(i,j)$ using the empirical distributions. To evaluate the accuracy, we estimate the causal effect with the midpoint of bounds and calculate the average error of each group. The results are shown in \cref{fig:finitesample_error}. 
Our method has a smaller average error than the Tian-Pearl bounds for the cases $H(Z)\leq 0.8$. For $H(Z)\leq 0.2$, the average error drops rapidly as the number of samples increases. This demonstrates that our method improves the causal effect estimation with finite data.

\begin{figure}[t]
    \centering
    \subfloat[$\abs{X}=2, \abs{Y}=2$]{\includegraphics[width=0.33\linewidth]{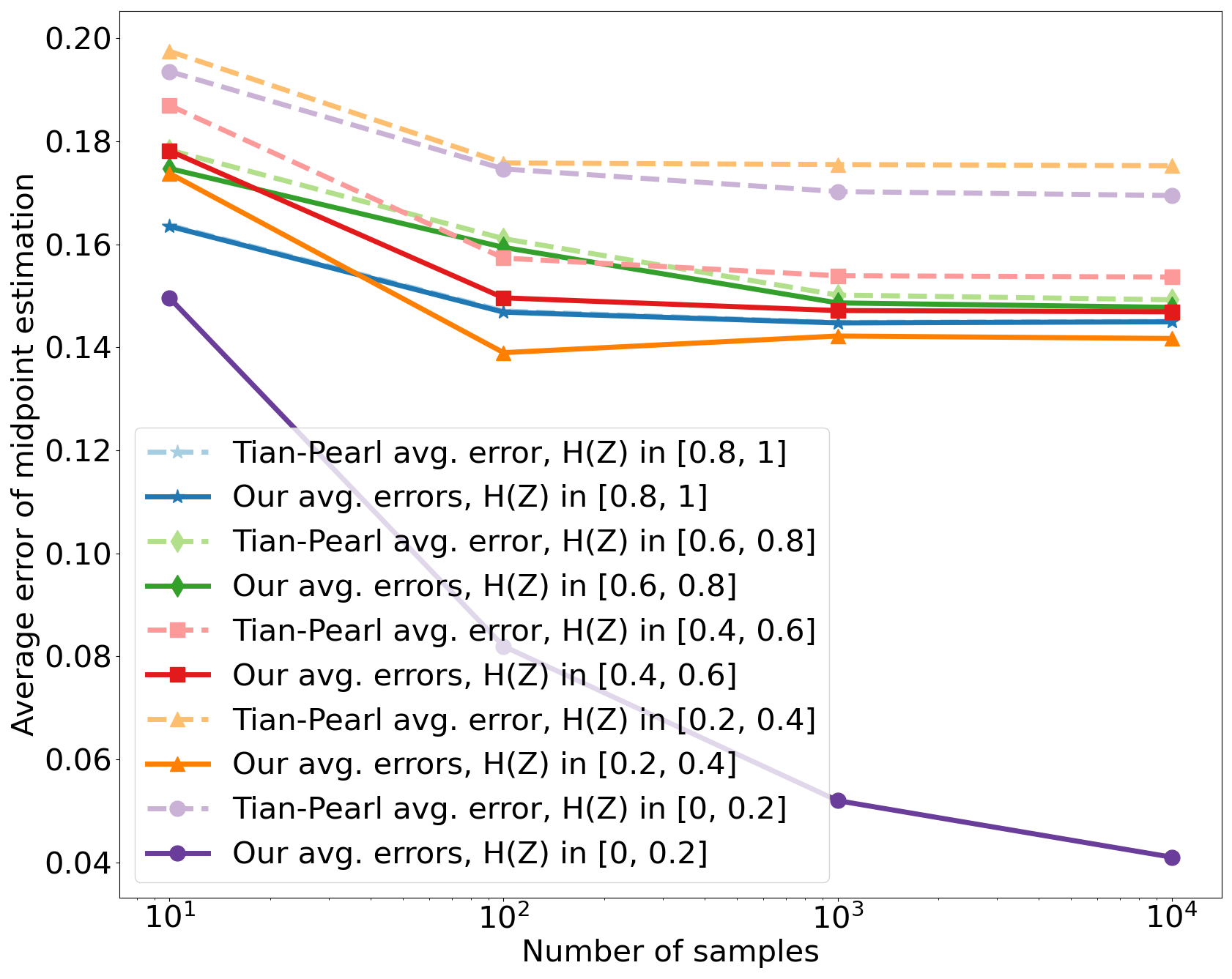}}
    \hfill
    \subfloat[$\abs{X}=2, \abs{Y}=10$]{\includegraphics[width=0.33\linewidth]{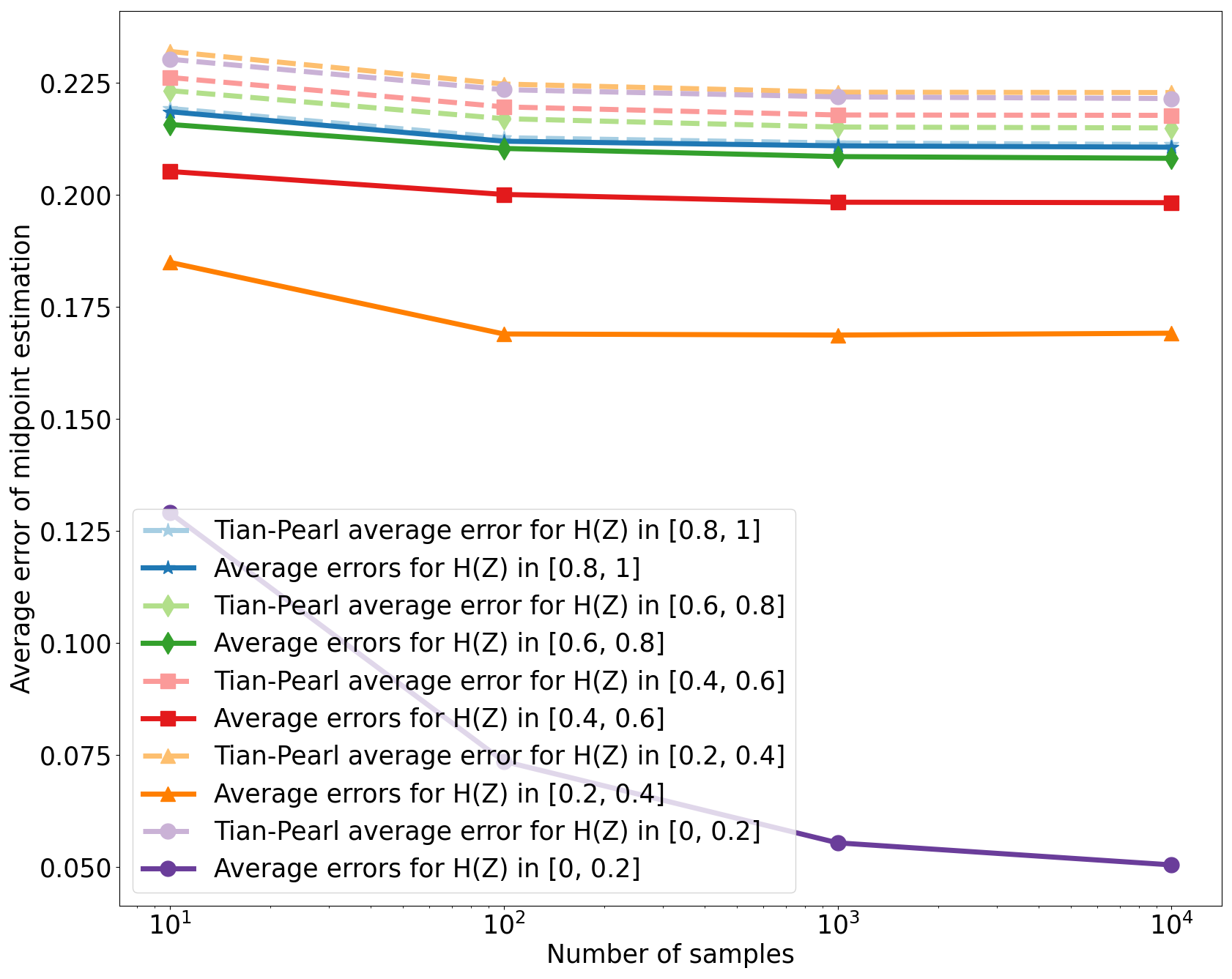}}
    \hfill
    \subfloat[$\abs{X}=2, \abs{Y}=10$]{\includegraphics[width=0.33\linewidth]{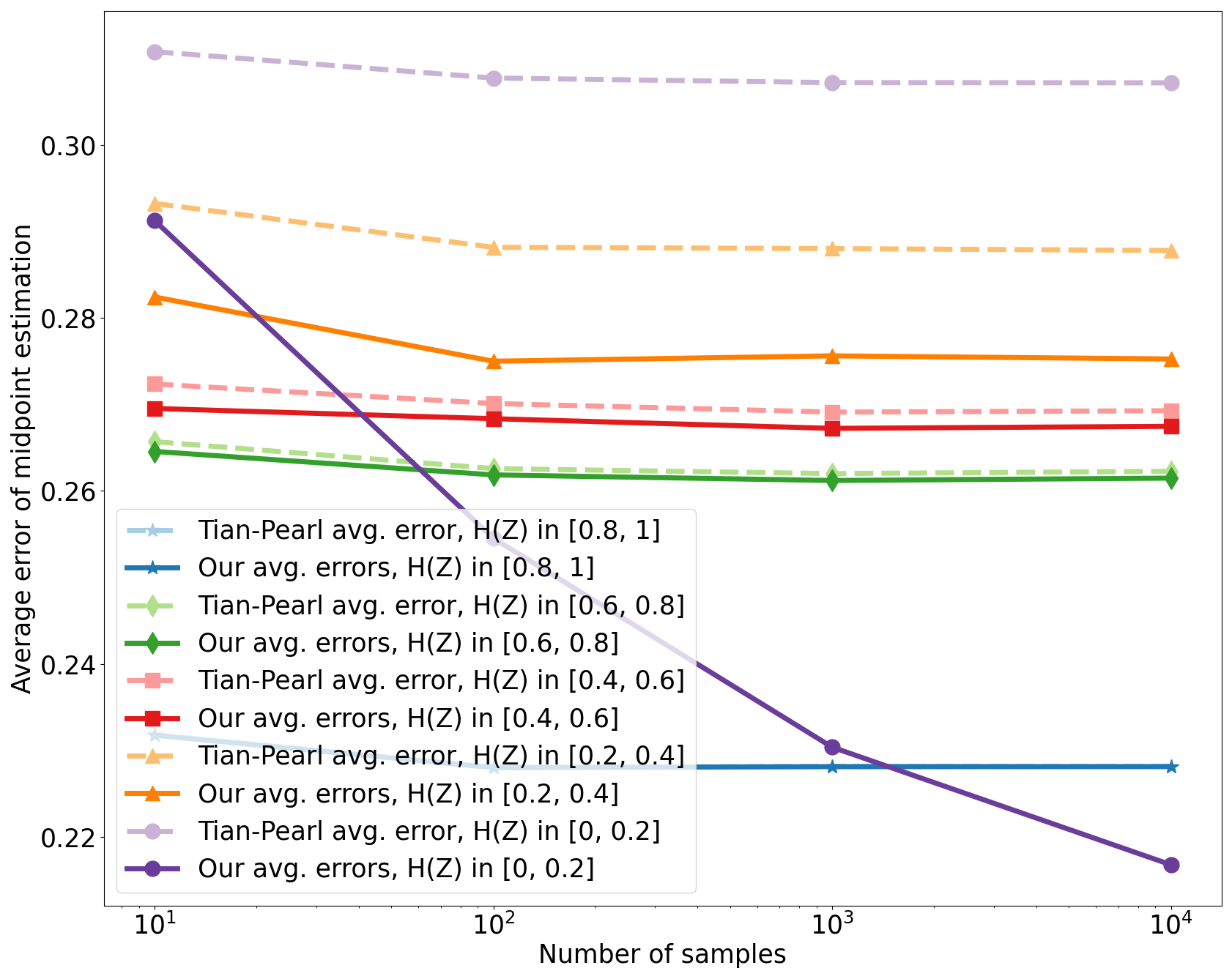}}
    \caption{The average error of midpoint estimation with finite samples. The dashed lines are average errors with the Tian-Pearl midpoint estimation, and the solid lines are the average error with our midpoint estimation.}
    \label{fig:finitesample_error}
\end{figure}

\subsection{Comparison of Formulations with Canonical Partition and the Counterfactual Distribution}
\label{subsec:compare}

\cref{eq:relate} demonstrates the equivalence of two bounds formulations for binary $X$ and $Y$ without entropy constraints. It can be extended to arbitrary discrete variables straightforwardly. For the entropy constraint in \cref{thm:cpalg}, we can envision a table similar to \cref{table:cf_conditional} where $p$-th column and $q$-th row are divided into multiple columns and rows. Intuitively, \cref{thm:cpalg} is an over-parameterization version of \cref{thm:cfalg}. In terms of partial identification, those two methods are identical. We verify this with an experiment similar to \cref{subsec:random}. We apply both methods with $\abs{X}, \abs{Y} = \{2, 4, 8\}$. For each case, we generate $100$ distributions and compute the bounds $P(y_i|do(x_j))$ for each pair of $(i,j)$. The experiments indicate that the two approaches have the same optimal values within a precision of three decimal places. The \cref{table:compare} provides the number of parameters and average runtime for the two methods. The counterfactual formulation is significantly more efficient when $\abs{X}$ is large.

\input{complexity}

\section{Conclusion}
In this paper, we proposed a way to utilize entropy to estimate the bounds of the causal effect. We formulate optimization problems with counterfactual probability, which significantly reduces the number of parameters in the optimization problem. We demonstrate a method to compute the entropy threshold easily so that we can use the entropy threshold as a criterion for applying entropy constraint. For the real-world problem, if we know that two variables are confounded by a confounder with entropy no more than the entropy threshold, we can apply the method and obtain tighter bounds. Another possible scenario where our method can be applied is when the distribution of the confounder is not provided as a joint distribution. Instead, only the marginal distributions $P(Z)$, $P(X, Y)$ are available, as in the example presented by Li and Pearl \yrcite{li2022bounds}. In such cases, our method can be utilized to derive tighter bounds.

Our optimization methods work for any discrete $X, Y$. Only the computation of the entropy threshold requires either $X$ or $Y$ being binary. For future works, it would be worthwhile to explore the tight-bounds condition for non-binary $X, Y$. To obtain entropy thresholds that scale with the number of states of the observed variables, one might need to consider the dependence between different queries $P(y|do(x_j))$ for $j\in \abs{X}$.

\section*{Acknowledgements}
This research has been supported in part by NSF Grant CAREER 2239375.

\bibliography{paper}
\bibliographystyle{icml2023}
\balance
\newpage
\appendix
\onecolumn

\section{Proof of \cref{thm:cpalg}}\label{pf:cpalg}
Recall the \cref{thm:cpalg}.
\cpalg*
\begin{proof}
    To show the LB and UB bound the causal effect, we first need to show the causal effect lies in the feasible set of the optimization problem. 

    Let $(R_y, R_x)$ be a canonical partition of $(X,Y)$. Let $a_{ij}$ = $P(R_y= i|R_x=j)$. By the construction of finite response variables, each $P(y_i|x_q)$ equal to the sum of $m^{n-1}$ terms of $P(R_y|R_x)$. Let $D_k$ be a set of indices such that $\sum_{i\in D_k} P(R_y=i|R_x=q) = P(y_i|x_q)$ and $\sum_{i\in D} P(R_y=i|R_x=q) = 1$ where $D = \bigcup_{i\in[n]} D_i$. Then we have the following relations. 
    \begin{align*}
        &\sum_{ij} a_{ij}P(x_j) = \sum P(R_y, R_x) = 1\\
        &\sum_{i=0}^{m^{n-1}} a_{iq} P(x_q) = \sum_{i\in D_k} P(R_y=i|R_x=q)P(x_q) = P(y_p,x_q)\\
        & \sum_{i,j}  a_{ij}P(x_j) \log\left(\frac{a_{ij}}{\sum_k a_{ik}P(x_k)}\right)\\
        &= I(R_x;R_y) \leq \theta.
    \end{align*}
    Since $R_x$ and $R_y$ are d-separated by the confounder, by the data processing inequality, the mutual information between $R_x, R_y$ is less than the entropy of the confounder. So the last inequality holds. Therefore we have $P(y_0|do(x_0))$ is in the feasible set.\\
    Since mutual information is a convex function of the conditional distributions, the set of $a_{ij}$ satisfies  $I(R_x; R_y)\leq \theta$ is convex. The objective function and all other constraints are linear functions of $a_{ij}$, so the optimization problem is convex and obtains global optimal in the feasible set. 
\end{proof}
    We use the CVXPY package to solve the problem and formulate the constraint according to the Disciplined Convex Programming rules.

 \section{Proof of \cref{thm:cfalg}}\label{pf:cfalg}
 Recall the \cref{thm:cfalg}.
 \cfalg*
\begin{proof}
    To show the LB and UB bound the causal effect, we first need to show the causal effect lies in the feasible set of the optimization problem. 

    Let $P(Y_{x_q}, X)$ be the counterfactual distribution for $x_q\in X$. Let $b_{ij} = P(Y_{x_q}=y_i|x_j)$, Then we have the following
    \begin{align*}
        &\sum_{ij} b_{ij}P(x_j) = \sum P(R_y, R_x) = 1\\
        \\
        &b_{iq}P(x_q) = P(Y_x=y_i|x_n)P(x_q) = P(y_i, x_q)\;\forall i\\
        \\
        & \sum_{i,j}  b_{ij}P(x_j) \log\left(\frac{b_{ij}}{\sum_k b_{ik}P(x_k)}\right)\\
        &= I(Y_x;X) \leq \theta.
    \end{align*}
    Since $Y_x$ and $X$ are d-separated by the confounder, by the data processing inequality, the mutual information between them is less than the entropy of the confounder. So the last inequality holds. Therefore we have $P(y_0|do(x_0))$ in the feasible set.
    
    Since mutual information is a convex function of the conditional distributions, the set of $b_{ij}$ satisfies  $I(Y_x; X)\leq \theta$ is convex. The objective function and all other constraints are linear functions of $b_{ij}$, so the optimization problem is convex and obtains global optimal in the feasible set. 

\end{proof} 
 We use the CVXPY package to solve the problem and formulate the constraint according to the Disciplined Convex Programming rules. 
 
\section{Proof of \cref{lm:binaryxy}} \label{pf:binaryxy}
Recall the \cref{lm:binaryxy}
\binaryxy*
\begin{proof}
By the law of total probability, we have that 
$$P(Y_x = y) = P(Y_x=y| x)P(x) + P(Y_x=y| x')P(x'),$$ 
and similarly 
$$P(Y_x = y') = P(Y_x=y'| x)P(x) + P(Y_x=y'|x')P(x').$$
From the observational distribution, we have $ P(Y_x=y|x)= P(y|x)$ , $ P(Y_x=y'|x) = P(y'|x)$. Denote $p = P(Y_x=y|x')$, $1-p = P(Y_x=y'|x')$.\\

We first show the case $P(y'|x)\leq P(y|x)$.\\
($1\implies 2$) Assume $P(Y_x=y)$ attain the Tian-Pearl lower bound, i.e. $ P(Y_x=y) = P(y,x)$. Since $P(Y_x=y|x) = P(y|x)$, we have $P(Y_x=y|x')P(x')=0$. Since $P(x')>0$,  $P(Y_x=y|x')=0$, so  $P(Y_x=y'|x')=1$. Then we have $P(Y_x=y') = P(Y_x=y'|x)P(x) + P(x') = 1-P(x, y)$ attain the Tian-Pearl upper bound. Thus $1\implies 2$.

($2\implies 3$) Assume $P(Y_x=y')$ attain the Tian-Pearl upper bound, we have $P(Y_x=y'|x')=1$ and $P(Y_x=y|x') =0$. We want to show that the mutual information is maximized when $p=1$. Since $I(Y_x; X)$ is a convex function of $P(Y_x|X)$, it is a convex of $p$. $I(Y_x;X) =0$ when $p=P(Y_x=y|x)$, and monotonically increasing for both $p>P(Y_x=y|x)$ and $p<P(Y_x=y|x)$. So $I(Y_x; X)$ obtains the local maximum at two boundaries $p=0, 1$. To compare those two points, denote $I(Y_x; X)$  as the mutual information if $p=0$, and $I'$ as the mutual information if $p=1$. Then we have $I - I' = P(x')\left(\log\frac{P(x')}{1+P(y'|x)} -\log\frac{P(x')}{1+P(y|x)}\right)\leq 0$,  since $P(y'|x)\leq P(y|x)$. The global maximum of mutual information is at $p =P(Y_x=y|x')=1$. 

($3\implies 1$) Assumes $I(Y_x;X)$ attain maximum given $P(X,Y)$. The above argument shows $P(Y_x=y|x') = 1$. So $P(Y_x=y) = P(x) + P(x,y)$ attain the Tian-Pearl upper bound.
\end{proof}

\section{Proof of \cref{lm:binaryx}} \label{pf:binaryx}

Recall \cref{lm:binaryx}
\binaryx*
\begin{proof}
    Given $P(Y,X)$, we have $P(Y_x = y_i|x_i)  = P(y_i|x)$ for all $i\leq n$. If $P(Y_x = y_p|x') = 1$, then $P(Y_x = y_p)$ attain the Tian-Pearl upper bound:
    $$P(Y_x = y_p) = P(Y_x = y_p|x)P(x) + P(Y_x = y_p|x')P(x') =  P(y_p,x) + P(x') = 1 - \sum_{i\neq p}P(y_i,x).$$
    
    Next, show the minimum mutual information that attains the Tian-Pearl lower bound. $P(Y_x=y_i) =  P(Y_x=y_i|x)P(x) + P(Y_x=y_i|x')P(x')$ attain the Tian-Pearl lower bound if $P(Y_x=y_i|x') = 0$ for all $i\neq p$.\\
    Since we fixed $P(Y_x|x) = P(Y|x)$, the domain of the mutual information is to a $(n-1)$simplex $\Delta^{n-1}$ of $P(Y_x|x')$. Since $I(Y_x; X)$ is convex with respect to $P(Y_x|X)$, this restricted function is also convex. Clearly, the restricted function obtains minimum when $P(Y_x|x') = P(Y_x|x)$. Since we fixed $P(y_p|x') = 0$, this corresponding to the restricted function on the $(n-2)$-simplex. With a similar argument, this restricted function is also convex. Now we only need to find the local extrema on the $(n-2)$-simplex.
    
    Let $P(Y_x=y_p|x') = 0$, and denote $P(y_i|x) = \alpha_i$ for all $i\leq n$ and $P(Y_x=y_i|x') = \beta_i$ for all $1\leq i \leq n$. So $P(Y_x) = [\alpha_0P(x), \alpha_1P(x)+\beta_1(P(x'), \dots , \alpha_nP(x)+ \beta_nP(x')].$\\
    Using the grouping property of entropy, we can write entropy as
    \begin{align*}
        H(Y_x) &=  H_b\left(\alpha_0P(x)\right) +H\left(\frac{\alpha_1P(x)+\beta_1P(x')}{1-\alpha_0P(x)}, \dots, \frac{\alpha_nP(x)+\beta_nP(x')}{1-\alpha_0P(x)}\right) \left(1-\alpha_0P(x)\right)\\
        &= H_b\left(\alpha_0P(x)\right) + H_b\left(\frac{\alpha_1P(x)+\beta_1P(x')}{1-\alpha_0P(x)}\right) \left(1-\alpha_0P(x)\right) \\
        & \qquad + H\left(\frac{\alpha_2P(x)+\beta_2P(x')}{1-\alpha_0P(x)}, \dots, \frac{\alpha_nP(x)+\beta_nP(x')}{1-\alpha_0P(x)}\right)\left(\frac{\sum_{i=2}^n(\alpha_iP(x)+\beta_iP(x'))}{1-\alpha_0P(x)}\right)
    \end{align*}

    Similarly, we can write the conditional entropy as
    \begin{align*}
        H(Y_x|X) &= P(x)H(Y_x|x) - P(x')H(Y_x|x') \\
        &= P(x)H(Y_x|x) - P(x')H(\beta_1, \dots, \beta_n) \\
        &= P(x)H(Y_x|x) - P(x')H_b(\beta_1) - P(x')H\left(\frac{\beta_2}{\sum_{i=2}^n\beta_i},\dots,\frac{\beta_n}{\sum_{i=2}^n\beta_i}\right)P\left(\sum_{i=2}^n\beta_i\right) 
    \end{align*}
    
    the mutual information as

    \begin{align*}
        I(Y_x; X) &= H(Y_x) - H(Y_x|X) \\
        &= H_b\left(\alpha_0P(x)\right) + H_b\left(\frac{\alpha_1P(x)+\beta_1P(x')}{1-\alpha_0P(x)}\right) \left(1-\alpha_0P(x)\right) \\
        & \qquad + H\left(\frac{\alpha_2P(x)+\beta_2P(x')}{1-\alpha_0P(x)}, \dots, \frac{\alpha_nP(x)+\beta_nP(x')}{1-\alpha_0P(x)}\right)\left(\frac{\sum_{i=2}^n(\alpha_iP(x)+\beta_iP(x'))}{1-\alpha_0P(x)}\right) \\
        & \qquad - P(x)H(Y_x|x) - P(x')H_b(\beta_1) - P(x')H\left(\frac{\beta_2}{\sum_{i=2}^n\beta_i},\dots,\frac{\beta_n}{\sum_{i=2}^n\beta_i}\right)P\left(\sum_{i=2}^n\beta_i\right) \\
    \end{align*}
    Now denote terms that do not involve $\beta_1$ as some constant. We can write the mutual information as follows.
    $$ I(Y_x; X) = C_1 +  \left(1-\alpha_0P(x)\right)H_b\left(\frac{\alpha_1P(x)+\beta_1P(x')}{1-\alpha_0P(x)}\right) + C_2 - C_3 - P(x')H_b(\beta_1) -C_4$$
    Then take the derivative with respect to $\beta_1$ and get
    \begin{align*}
        \frac{\partial I(Y_x; X)}{\partial \beta_1} &=  \left(1-\alpha_0P(x)\right)\left(\log\frac{1-\alpha_0P(x) - (\alpha_1P(x)+\beta_1P(x')) }{\alpha_1P(x)+\beta_1P(x')}\right)\frac{P(x')}{1-\alpha_0P(x)} - P(x') \log\frac{1-\beta_1}{\beta_1}\\
        &= P(x') \left(\log\frac{1-(\alpha_0+\alpha_1)P(x) +\beta_1P(x') }{\alpha_1P(x)+\beta_1P(x')} - \log\frac{1-\beta_1}{\beta_1}\right)
    \end{align*}

    Then we can find the local extrema by setting the derivative to zero. 
    \begin{align*}
         \frac{\partial I(Y_x; X)} {\partial \beta_1} &= 0\\
         P(x') \log\frac{1-(\alpha_0+\alpha_1)P(x) - \beta_1P(x')}{\alpha_1P(x)+\beta_1P(x')} &=  P(x')\log\frac{1-\beta_1}{\beta_1}\\
         \log\frac{1-(\alpha_0+\alpha_1)P(x) - \beta_1P(x') }{\alpha_1P(x)+\beta_1P(x')}  &= \log\frac{1-\beta_1}{\beta_1}\\
         \frac{1-(\alpha_0+\alpha_1)P(x) - \beta_1P(x') }{\alpha_1P(x)+\beta_1P(x')} &= \frac{1-\beta_1}{\beta_1}\\
          (\alpha_1P(x)+\beta_1P(x'))(1-\beta_1)  &= (1-(\alpha_0+\alpha_1)P(x) -\beta_1P(x'))\beta_1\\
         \alpha_1P(x)- \beta_1\alpha_1P(x) +\beta_1P(x') &= (1-(\alpha_0+\alpha_1)P(x)) \beta_1\\
          (1-(\alpha_0+\alpha_1)P(x) +\alpha_1P(x) - P(x')) \beta_1 &=\alpha_1P(x) \\
         (P(x)-(\alpha_0+\alpha_1)P(x) +\alpha_1P(x)) \beta_1&= \alpha_1P(x) \\
         (1 - \alpha_0-\alpha_1+ \alpha_1)P(x)  \beta_1&= \alpha_1P(x) \\
        \beta_1 &= \frac{\alpha_1}{1-\alpha_0}
    \end{align*}

    Repeat the steps for $1\leq i \leq n$, we can get the local minimum at $\beta_i = \frac{\alpha_i}{1-\alpha_0}$ for all $1\leq i \leq n$. Since the mutual information is convex, these points give the global minimum of mutual information.
\end{proof}

\section{Proof of \cref{lm:binaryy}}\label{pf:binaryy}
Recall the \cref{lm:binaryy}
\binaryy*

\begin{proof}
     Given $P(Y,X)$, we have $P(Y_x = y|x_q) = P(y|x_q)$ for all $y\in Y$. Assumes $P(Y_{x_q}=y)$ attain the Tian-Pearl upper bound, i.e. $$P(Y_{x_q}=y) = 1-  P(y',x_q) = P(y, x_q) + \sum_{j\neq q} (P(y, x_j)+P(y', x_j)) = P(y_p, x_q) + \sum_{j\neq q}P(x_j).$$
    On the other hand, we have 
    $$P(Y_{x_q} = y_p) = \sum_{j}P(Y_{x_q}=y_p| x_j)P(x_j) .$$
    Combines the above two equations, we get $P(Y_{x_q}=y_p| x_j) = 1$ for all $j\neq q$.

    For the lower bound, assumes $P(Y_{x_q}=y_p) = P(y_p,x_q)$ by a similar argument as above, we have 
    $$P(Y_{x_q} = y_p) = \sum_{j}P(Y_{x_q}=y_p| x_j)P(x_j) = P(y_p, x_q) + \sum_{j\neq q} P(Y_{x_q} = y_p | x_j)P(x_j)$$
    So from the above two equations, we get $P(Y_{x_q}=y_p| x_j) = 0$ for all $j\neq q$.
\end{proof}

\section{Proof of \cref{thm:entropyconstraint}}\label{pf:entropyconstraint}
Recall the \cref{thm:entropyconstraint}
\entropyconstraint*

\begin{proof}
     Let $P(U, V)$ be the constructed joint distribution according to the theorem. By \cref{lm:binaryxy}, assuming $P(y'|x)\leq P(y|x)$, $I(U;V)$ is maximum is equivalent to $P(v_0) = P(v_0|u_0)P(u_0) + P(v_0|u_1)P(u_1)$ attain maximum or minimum. That is when $P(v_0|u_1) = 1$ or $P(v_1|u_1) =1$
    
    If $P(v_0|u_1) = 1$, 
    \begin{equation}\label{eq:condition_ub}
        I(U;V) = H(V) - H(V|U) = H_b((1-P(y_p|x_q))P(x_q)) - P(x_q)H_b(P(y_p|x_q))
    \end{equation}
    If $P(v_1|u_1) = 1$, 
    \begin{equation}\label{eq:condition_lb}
        I(U;V) = H(V) - H(V|U) = H_b(P(y_p|x_q)P(x_q)) - P(x_q)H_b(P(y_p|x_q))
    \end{equation}
    
    First, consider the case where $Y$ is a binary variable and $\abs{X} = n$. By \cref{lm:binaryy}, $P(Y_{x_q}=y)$ attain the Tian-Pearl upper bound when $P(Y_{x_q}=y|x_j) = 1 $ for all $j\neq q$. So we have 
    $$  Y_x = \begin{cases} y &P(y|x_0)P(x_0) +\sum_{j=1}^n P(x_j) \\
    y' &P(y'|x_0)P(x_0)\\
    \end{cases}.$$
    Since $P(Y_{x_q}=y|x_j) = 1 $ for all $j\neq q$, $H(Y_{x_q}|x_j) = 0$ for all $j\neq q$. So $H(Y_{x_q}|X) = P(x_q)H_b(P(y|x_q))$. Then we have 
    $$I(Y_{x_q};X) = H(Y_{x_q}) - H(Y_{x_q}|X) = H_b(P(y'|x_q)P(x_q)) - P(x_q)H_b(P(y|x_q)).$$
    This equals to the \cref{eq:condition_ub}, so we have $P(Y_{x_q}=y)$ attain the Tian-Pearl upper bound implies $I(U; V)$ obtains maximum. 

    Again by \cref{lm:binaryy}, $P(Y_{x_q}=y)$ attain the Tian-Pearl lower bound when $P(Y_{x_q}=y'|x_j) = 1 $ for all $j\neq q$. So we have 
    $$  Y_x = \begin{cases} y &P(y|x_0)P(x_0) \\
    y' &P(y'|x_0)P(x_0) +\sum_{j=1}^n P(x_j).\\
    \end{cases}$$
    Since $P(Y_{x_q}=y'|x_j) = 1 $ for all $j\neq q$, $H(Y_{x_q}|x_j) = 0$ for all $j\neq q$. So $H(Y_{x_q}|X) = P(x_q)H_b(P(y|x_q))$. Then we have 
    $$I(Y_{x_q};X) = H(Y_{x_q}) - H(Y_{x_q}|X) = H_b(P(y|x_q)P(x_q)) - P(x_q)H_b(P(y|x_q)).$$
    This equals to the \cref{eq:condition_lb}, so we have $P(Y_{x_q}=y)$ attains the Tian-Pearl lower bound implies $I(U; V)$ obtains maximum. 

    We have shown for the binary $Y$,the causal effect $P(Y_x)$ attains Tian-Pearl bounds implies  $I(Y_x; X) = \max{(I(U; V))}$. Suppose we have $I(Y_x; X) \leq H(Z) < \max(I(U; V))$, by the contraposition, $P(Y_x)$ cannot attains Tian-Pearl bounds.

    Now consider the case where $X$ is a binary variable and $\abs{Y} = m$. By \cref{lm:binaryx}, the causal effect $P(Y_x=y_p)$ attains Tian-Pearl upper bound when $P(Y_x=y_p|x') = 1$; attains lower bound with minimum mutual information when $P(Y_x=y_i|x') = \frac{P(Y_x=y_i|X=x)}{\sum_{j\neq p}P(Y=y_j|X=x)}$ for all $i\neq p$.

    For the upper bound case, assuming $P(Y_x=y_p|x') = 1$, we have $P(Y_x=y_i|x') =0$ and $H(X|y_i) = 0$ for all $i\neq p$. $H(X|Y) = P(y_p)H(X|y_p)$.

    The mutual information is $$I(Y_x; X) = H_b(x) - P(y_p)H(X|y_p). $$ 

    On the other hand, we can write \cref{eq:condition_ub} as 

    $$ I(U;V) = H(U) - H(U|V) = H_b(x) - P(y_p)H(X|y_p) = I(Y_x; X).$$
    So we have $P(Y_x)$ attains the Tian-Pearl lower bound implies $I(Y_x;X) = \max(I(U; V))$

    Next assuming $P(Y_x=y_i|x') = \frac{P(Y_x=y_i|X=x)}{\sum_{j\neq p}P(Y=y_j|X=x)}$ for all $i\neq p$. We have $P(Y_x=y_p|x)=0$. Denote $P(Y_x=y_i|x) = \alpha_i$. Using the grouping property of entropy, we could get
    \begin{align*}
        H(Y_x|X) &= P(x)H(Y_x|x) + P(x')H(Y_x|x') \\
        &= P(x)H(\alpha_0,\dots,\alpha_n) + P(x')H\left(\frac{\alpha_0}{1-\alpha_p},\dots, \frac{\alpha_{p-1}}{1-\alpha_p}, \frac{\alpha_{p+1}}{1-\alpha_p},\dots, \frac{\alpha_m}{1-\alpha_p}\right) \\
        &= P(x) \left[ H(\alpha_p) + (1-\alpha_p)H\left(\frac{\alpha_0}{1-\alpha_p},\dots, \frac{\alpha_{p-1}}{1-\alpha_p}, \frac{\alpha_{p+1}}{1-\alpha_p},\dots, \frac{\alpha_m}{1-\alpha_p}\right) \right] \\
        &\qquad + P(x')H\left(\frac{\alpha_0}{1-\alpha_p},\dots, \frac{\alpha_{p-1}}{1-\alpha_p}, \frac{\alpha_{p+1}}{1-\alpha_p},\dots, \frac{\alpha_m}{1-\alpha_p}\right) \\
        &= P(x) H(\alpha_p) + \left(P(x)(1-\alpha_p)+P(x')\right) H\left(\frac{\alpha_0}{1-\alpha_p},\dots, \frac{\alpha_{p-1}}{1-\alpha_p}, \frac{\alpha_{p+1}}{1-\alpha_p},\dots, \frac{\alpha_m}{1-\alpha_p}\right)\\
        &= P(x) H(\alpha_p) + \left(1-\alpha_pP(x)\right) H\left(\frac{\alpha_0}{1-\alpha_p},\dots, \frac{\alpha_{p-1}}{1-\alpha_p}, \frac{\alpha_{p+1}}{1-\alpha_p},\dots, \frac{\alpha_m}{1-\alpha_p}\right).\\
    \end{align*}

    Then we have 
    $$  Y_x = \begin{cases} y_0 & \alpha_0P(x) + \frac{\alpha_0}{1-\alpha_p}P(x') \\
                            \vdots & \vdots\\
                            y_p & \alpha_pP(x)\\
                            \vdots & \vdots\\
                            y_m & \alpha_mP(x) + \frac{\alpha_m}{1-\alpha_p}P(x')
            \end{cases}$$
    Again by the grouping property, we have
    \begin{align*}
        H(Y_x) &= H_b(\alpha_pP(x)) + (1-\alpha_pP(x)) H\left(\frac{\alpha_0P(x) + \frac{\alpha_0}{1-\alpha_p}P(x')}{1-\alpha_pP(x)}, \dots\right) \\
        &=  H_b(\alpha_pP(x)) + (1-\alpha_pP(x)) H\left(\frac{ \frac{\alpha_0P(x)(1-\alpha_p) +\alpha_0P(x')}{1-\alpha_p}}{1-\alpha_pP(x)}, \dots\right) \\
        &=  H_b(\alpha_pP(x)) + (1-\alpha_pP(x)) H\left(\frac{ \frac{\alpha_0P(x)(1-\alpha_p) +\alpha_0(1-P(x))}{1-\alpha_p}}{1-\alpha_pP(x)}, \dots\right) \\
        &=  H_b(\alpha_pP(x)) + (1-\alpha_pP(x)) H\left(\frac{ \frac{\alpha_0P(x)-\alpha_0\alpha_pP(x) + \alpha_0-\alpha_0P(x)}{1-\alpha_p}}{1-\alpha_pP(x)}, \dots\right) \\
        &=  H_b(\alpha_pP(x)) + (1-\alpha_pP(x)) H\left(\frac{\alpha_0-\alpha_0\alpha_pP(x) }{(1-\alpha_p)(1-\alpha_pP(x))}, \dots\right) \\
        &=  H_b(\alpha_pP(x)) + (1-\alpha_pP(x)) H\left(\frac{\alpha_0(1-\alpha_pP(x)) }{(1-\alpha_p)(1-\alpha_pP(x))}, \dots\right) \\
        & = H_b(\alpha_pP(x)) +\left(1-\alpha_pP(x)\right) H\left(\frac{\alpha_0}{1-\alpha_p},\dots, \frac{\alpha_{p-1}}{1-\alpha_p}, \frac{\alpha_{p+1}}{1-\alpha_p},\dots, \frac{\alpha_m}{1-\alpha_p}\right)
    \end{align*}

    Finally, we have
     \begin{align*}
         I(Y_x;X) &= H(Y_x) - H(Y_x|X) \\
         &= H_b(\alpha_pP(x)) +\left(1-\alpha_pP(x)\right) H\left(\frac{\alpha_0}{1-\alpha_p},\dots, \frac{\alpha_{p-1}}{1-\alpha_p}, \frac{\alpha_{p+1}}{1-\alpha_p},\dots, \frac{\alpha_m}{1-\alpha_p}\right) \\
         &\qquad - P(x) H_b(\alpha_p) + \left(1-\alpha_pP(x)\right) H\left(\frac{\alpha_0}{1-\alpha_p},\dots, \frac{\alpha_{p-1}}{1-\alpha_p}, \frac{\alpha_{p+1}}{1-\alpha_p},\dots, \frac{\alpha_m}{1-\alpha_p}\right)\\
         &= H_b(\alpha_pP(x))- P(x) H_b(\alpha_p)\\
         &= H_b(P(y_p|x_q)P(x_q)) - P(x_q)H_b(P(y_p|x_q))
     \end{align*}
    This equals to \cref{eq:condition_lb}. So the minimum $I(Y_x;X)$ for $P(Y_x=y_p)$ attains Tian-Pearl lower bound is equal to the maximum of $I(U; V)$. For any other distribution where $P(Y_x)$ attains Tian-Pearl lower bound has mutual information greater than $\max{(I(U; V))} $. Hence $P(Y_x)$ attains Tian-Pearl lower bound implies the $I(Y_x;X) \geq \max{(I(U; V))}$.
    
    We have shown that for the binary $X$, the causal effect $P(Y_x)$ attains Tian-Pearl bounds implies  $I(Y_x; X) \geq \max{(I(U; V))}$. Suppose we have $I(Y_x; X) \leq H(Z) < \max(I(U; V))$, by the contraposition, $P(Y_x)$ cannot attains Tian-Pearl bounds.
\end{proof}

\section{Sampling the Joint Distribution}\label{sec:sampling}
Given a DAG as shown in \cref{subfig:dag}, we first generate $P(Z)\thicksim Dir(\alpha)$ for some small $\alpha$ value. In this experiment, we use $\alpha=0.1$. For $X$ with $n$ states, we first construct a vector $\bf{v} =\frac{1}{T}[1, \frac{1}{2},\dots, \frac{1}{n}]$, where $T$ is normalizing factor such that $\sum\bf{v}=1$. Then for each state of $Z$, we create a shifted $\bf{v_k}$ by rolling the values of $\bf{v}$. Then we sample $P(X|z_k)\thicksim Dir(\bf{v_k})$. Similarly, for $Y$ with $m$ states, we construct a vector $\bf{u} =\frac{1}{T}[1, \frac{1}{2},\dots, \frac{1}{m}]$ and for each $x_j, z_k$, we sample $P(Y|x_j, z_k)\thicksim Dir(\bf{u_i})$.
This procedure was described by Chickering and Meek \yrcite{chickering2012finding}. They use this method to prevent parent-child relationships between nodes from being uniform for a given DAG.

\clearpage

\section{Additional Experiment on ASIAN dataset}

\begin{figure}[htb]
    \centering
    \includegraphics[scale=0.18]{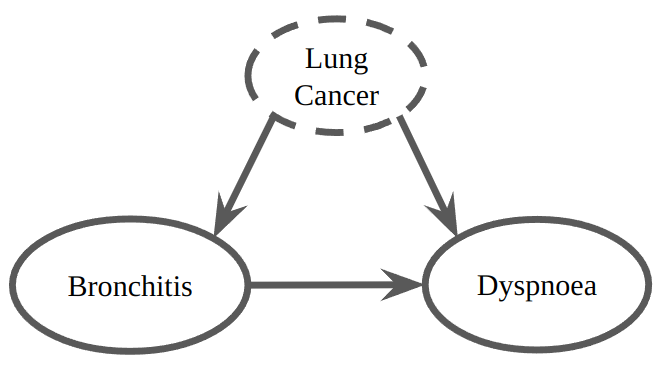}
    \caption{Causal graph for a subset of features from the ASIAN dataset. Unused variables are omitted.}
    \label{fig:asian_graph}
\end{figure}

We experiment with the ASIA dataset \cite{lauritzen1988local}. The causal graph is shown in \cref{fig:asian_graph}. We compute the bounds of the causal effect of Bronchitis on Dyspnoea with lung cancer as a confounder. The results are summarized in \cref{table:asian}. In this example, Bronchitis and Dyspnoea are connected through a backdoor path consisting of ``smoke'' and ``lung cancer''. In such a case, we can use the confounder with small entropy as a constraint to get tighter bounds, even if other variables in the backdoor path are more complex. The improvement of bounds shows the causal relationship between variables. The lower bound of $Dyspnoea|do(Bronchitis)$ increases from $0.364$ to $0.461$, and the upper bound of $Dyspnoea|do(Not Bronchitis)$ drops from $0.522$ to $0.412$. Since the bounds lie below $0.5$, one can be more certain that Bronchitis has some effect on Dyspnoea.

\input{asian_table}

\end{document}

%% file: adult_table.tex
\begin{table*}[tbh]
\footnotesize
\caption{Results of Causal Effect in INSURANCE and ADULT dataset}
\label{table:real_exp}
\vskip 0.15in
\begin{center}
\begin{sc}
\begin{tabular}{l |c |c c |c |c | r}
\toprule
Dataset & Subgroup &X &  Y & H(Z) & Our bounds & T-P bounds \\
\midrule
Insur & Under 5,000 miles, normal& Car Cost & Prop Cost & Acci \\ \cline{3-4}
&  & 100,000 & 10,000 & 0.092 & [0.000, \blue{0.246}] & [0.000, 0.800]\\
& & 100,000 & 100,000 & 0.092 & [\blue{0.699}, 0.996] & [0.196, 0.996]\\
& & 100,000 & 1,000,000 & 0.092 & [0.004, \blue{0.301}] & [0.004, 0.804]\\
& & 1,000,000 & 10,000 & 0.092 & [0.000, \blue{0.044}] & [0.000, 0.249]\\
& & 1,000,000 & 100,000 & 0.092 & [0.000, \blue{0.044}] & [0.000, 0.249]\\
& & 1,000,000 & 1,000,000 & 0.092 & [\blue{0.956}, 0.999] & [0.751, 0.999]\\
\midrule
Adult& & Relationship & Income & Age \\ \cline{3-4}
& below high school, full-time & Yes & $<=50$K & 0.21 & [\blue{0.605}, 0.934]& [0.423, 0.934] \\
& below high school, full-time & No & $<=50$K & 0.21 & [\blue{0.762}, 0.985]& [0.496, 0.985] \\
& below high school, full-time & Yes & $>50$K & 0.21 & [0.066, \blue{0.395}]& [0.066, 0.577] \\
& below high school, full-time & No & $>50$K & 0.21 & [0.015, \blue{0.238}]& [0.015, 0.504] \\
& above high school, part-time & Yes & $<=50$K & 0.41 & [\blue{0.186}, 0.903]& [0.183, 0.903] \\
& above high school, part-time & No & $<=50$K & 0.41 & [\blue{0.779}, 0.982]& [0.703, 0.983] \\
& above high school, part-time & Yes & $>50$K & 0.41 & [0.017, \blue{0.814}]& [0.096, 0.817] \\
& above high school, part-time & No & $>50$K & 0.41 & [0.017,\blue{0.220}]& [0.017, 0.297] \\
& above high school, full-time & Yes & $<=50$K & 0.12 & [\blue{0.310}, \blue{0.664}]& [0.250, 0.734] \\
& above high school, full-time & No & $<=50$K & 0.12 & [\blue{0.725}, 0.953]& [0.438, 0.953] \\
& above high school, full-time & Yes & $>50$K & 0.12 & [\blue{0.336}, \blue{0.690}]& [0.266, 0.750] \\
& above high school, full-time & No & $>50$K & 0.12 & [0.046, \blue{0.275}]& [0.046, 0.562] \\

\bottomrule
\end{tabular}
\end{sc}
\end{center}
\vskip -0.1in
\end{table*}

%% file: complexity.tex
\begin{table}[bth]
\footnotesize
\caption{Comparison of methods of \cref{thm:cpalg} and \cref{thm:cfalg}}
\label{table:compare}
\vskip 0.15in
\begin{center}
\begin{sc}
\begin{tabular}{|l |c |c |c |c |c | c |r}
\toprule
& & \multicolumn{2}{c|}{\cref{thm:cpalg}} & \multicolumn{2}{c|}{\cref{thm:cfalg}}  \\
$\abs{X}$& $\abs{Y}$& \vtop{\hbox{\strut Num of}\hbox{\strut Param}} & \vtop{\hbox{\strut Ave }\hbox{\strut time}}  &\vtop{\hbox{\strut Num of}\hbox{\strut Param}} & \vtop{\hbox{\strut Ave }\hbox{\strut time}}  \\
\midrule
 2 & 2 & 8 & 0.54s & 4 & 0.19s \\
 2 & 4 & 32 & 3.37s & 8 & 1.03s \\
 2 & 8 & 128 & 18.39s & 16 & 2.94s \\
 4 & 2 & 64 & 13.35s & 8 & 1.62s \\
 8 & 2 & 2048 & 1138.72s & 16 & 5.06s\\

\bottomrule
\end{tabular}
\end{sc}
\end{center}
\vskip -0.1in
\end{table}

%% file: asian_table.tex
\begin{table*}[htb]
\footnotesize
\caption{Results of Causal Effect in ASIAN dataset}
\label{table:asian}
\vskip 0.15in
\begin{center}
\begin{sc}
\begin{tabular}{l |c |c c |c |c | r}
\toprule
Dataset  &X &  Y & H(Z) & Our bounds & T-P bounds \\
\midrule
Asia & Bronc & Dysp & Cancer \\ \cline{3-4}
    & Yes & Yes & 0.31 & [\blue{0.461}, 0.914] & [0.364, 0.914] \\
    & Yes & No & 0.31 & [0.072, \blue{0.412}] & [0.072, 0.522] \\
    & No & Yes & 0.31 & [0.086, \blue{0.539}] & [0.086, 0.636] \\
    & No & No & 0.31 & [\blue{0.588}, 0.928] & [0.478, 0.928] \\

\bottomrule
\end{tabular}
\end{sc}
\end{center}
\vskip -0.1in
\end{table*}